\RequirePackage{fix-cm}
\documentclass[twocolumn]{svjour3}          
\smartqed  

\usepackage[utf8]{inputenc} 
\usepackage[T1]{fontenc}    
\usepackage{hyperref}       
\usepackage{url}            
\usepackage{booktabs}       
\usepackage{amsfonts}       
\usepackage{nicefrac}       

\usepackage{times}
\usepackage{epsfig}
\usepackage{graphicx}
\usepackage{amsmath}
\usepackage{amssymb}
\usepackage{amsbsy}

\usepackage{graphicx} 
\usepackage{subfigure}

\usepackage{algorithm, algpseudocode}
\let\oldReturn\Return
\renewcommand{\Return}{\State\oldReturn}
\usepackage{color}
\usepackage{url}
\usepackage{graphics}
\usepackage{amsmath}
\usepackage{amssymb}
\usepackage{enumitem}
\usepackage[mathscr]{eucal}
\usepackage{booktabs}
\usepackage{tabularx}
\usepackage{epstopdf}
\usepackage[all]{xy}
\usepackage{multirow}
\usepackage{multicol}

\usepackage{caption}

\usepackage{sidecap}

\def\a{\mathbf{a}}
\def\b{\mathbf{b}}

\def\g{\mathbf{g}}

\def\\v{\boldsymbol{\upsilon}}
\def\v{\mathbf{v}}
\def\w{\mathbf{w}}

\def\x{\mathbf{x}}
\def\y{\mathbf{y}}

\def\A{\mathbf{A}}
\def\B{\mathbf{B}}

\def\I{\mathbf{I}}
\def\M{\mathbf{M}}

\def\X{\mathbf{X}}

\def\bupsilon{\boldsymbol{\upsilon}}

\usepackage{algorithm}

\long\def\comment#1{}

\newcommand{\ie}{\textit{i.e.}}
\newcommand{\eg}{\textit{e.g.}}

\newcommand{\argmin}{\mathop{ \arg\!\min}}

\newcommand{\orcid}[1]{\href{https://orcid.org/#1}{\textcolor[HTML]{A6CE39}{\aiOrcid}}}

\usepackage{tikz}
\newcommand*\circled[1]{\tikz[baseline=(char.base)]{\node[shape=circle,draw,inner sep=0.6pt] (char) {#1};}}

\journalname{International Journal of Computer Vision}

\begin{document}

\title{MAP Inference via $\ell_2$-Sphere Linear Program Reformulation
\thanks{Baoyuan Wu was partially supported by Tencent AI Lab and King Abdullah University of Science
and Technology (KAUST). 
Li Shen was supported by Tencent AI Lab. 
Bernard Ghanem was supported by the King Abdullah University of Science and Technology (KAUST) Office of Sponsored Research (OSR).
Tong Zhang was supported by the Hong Kong University of Science and Technology (HKUST). 
Li Shen is the corresponding author.}
}

\author{\textbf{Baoyuan Wu}
\and
\textbf{Li Shen} \and 
\textbf{Tong Zhang} 
\and
\textbf{Bernard Ghanem} 
}

\institute{Baoyuan Wu \at
              Tencent AI Lab, Shenzhen 518000, China  \\
              \email{wubaoyuan1987@gmail.com}           
           \and
           Li Shen \at
           Tencent AI Lab, Shenzhen 518000, China \\
           \email{mathshenli@gmail.com}
           \and
           Tong Zhang \at Hong Kong University of Science and Technology, Hong Kong, China \\
           \email{tongzhang@tongzhang-ml.org}
           \and
           Bernard Ghanem \at King Abdullah University of Science and Technology, Thuwal 23955, Saudi Arabia \\
           \email{bernard.ghanem@kaust.edu.sa}
}

\maketitle

\begin{abstract}
Maximum a posteriori (MAP) inference is an important task for graphical models. Due to complex dependencies among variables in realistic models, finding an exact solution for MAP inference is often intractable. Thus, many approximation methods have been developed, among which the linear programming (LP) relaxation based methods show promising performance. 
However, one major drawback of LP relaxation is that it is possible to give fractional solutions.  
Instead of presenting a tighter relaxation, in this work we propose a continuous but equivalent reformulation of the original MAP inference problem, called LS-LP. We add the $\ell_2$-sphere constraint onto the original LP relaxation, leading to an intersected space with the local marginal polytope that is equivalent to the space of all valid integer label configurations. Thus, LS-LP is equivalent to the original MAP inference problem. 
We propose a perturbed alternating direction method of multipliers (ADMM) algorithm to optimize the LS-LP problem, by adding a sufficiently small perturbation $\epsilon$ onto the objective function and constraints. 
We prove that the perturbed ADMM algorithm globally converges to the $\epsilon$-Karush–Kuhn–Tucker ($\epsilon$-KKT) point of the LS-LP problem. 
The convergence rate will also be analyzed. 
Experiments on several benchmark datasets from Probabilistic Inference Challenge (PIC 2011) and OpenGM 2 show competitive performance of our proposed method against state-of-the-art MAP inference methods. 
\end{abstract}

\section{Introduction}
\label{sec: introduction}

Given the probability distribution of a graphical model, maximum a posteriori (MAP) inference aims to infer the most probable label configuration. MAP inference can be formulated as an integer linear program (ILP) \cite{variational-inference-book-2008}.
However, due to the integer constraint, the exact optimization of ILP 
is intractable in many realistic problems. 
To tackle it, a popular approach is relaxing ILP to a continuous linear program over a local marginal polytope, \ie, $\mathcal{L}_G$ (defined in Section \ref{sec: background}), called linear programming (LP) relaxation. 
The optimal solution to the LP relaxation will be obtained at the vertices of $\mathcal{L}_G$. 
It has been known \cite{variational-inference-book-2008} that all valid integer label configurations are at the vertices of $\mathcal{L}_G$, but not all vertices of $\mathcal{L}_G$ are integer, while some are fractional. Since LP relaxation is likely to give fractional solutions, the rounding method must be adopted to generate integer solutions. 
To alleviate this issue, intense efforts have been made to design tighter relaxations (\eg, high-order relaxation \cite{MAP-LP-thesis-2010}) based on LP relaxation, 
such that the proportion of fractional vertices of $\mathcal{L}_G$ can be reduced. 
However, the possibility of fractional solutions still exists. 
And, these tighter relaxations are often much more computationally expensive than the original LP relaxation. 
Moreover, there are also exact inference methods, such as branch-and-bound \cite{branch-and-bound-1960} and cutting-plane \cite{cutting-plane-1960}, which utilize LP relaxation as sub-routines, leading to much higher computational cost than approximate methods. 

Instead of proposing a new approximation with a tighter relaxation, we propose an exact reformulation of the original MAP inference problem. 
Specifically, we add a new constraint, called $\ell_2$-sphere \cite{WU-lpbox-ADMM-PAMI-2018}, onto the original LP relaxation problem. It enforces that the solution $\x \in \mathbb{R}^n$ should be on a $\ell_2$-sphere, \ie, $\parallel\x-\frac{1}{2}\parallel_2^2 = \frac{n}{4}$. 
We can prove that the intersection between the $\ell_2$-sphere constraint and the local polytope $\mathcal{L}_G$ is equivalent to the set of all possible label configurations of the original MAP inference problem, \ie, the constraint space of the ILP problem. 
Thus, the proposed formulation, dubbed LS-LP, is an equivalent but continuous reformulation of the ILP formulation for MAP inference. 
Furthermore, inspired by \cite{AD3-JMLR-2015} and \cite{WU-lpbox-ADMM-PAMI-2018}, we adopt the ADMM algorithm \cite{ADMM-boyd-2011}, to not only separate the different constraints, but also decompose variables to allow parallel inference by exploiting the factor graph structure. 
Although the $\ell_2$-sphere constraint is non-convex, we prove that the ADMM algorithm for the LS-LP problem with a sufficiently small perturbation $\epsilon$ will globally converges to the $\epsilon$-KKT \cite{KKT-1939,KKT-2014} point of the original LS-LP problem. 
The obvious advantages of the proposed LS-LP formulation and the corresponding ADMM algorithm include: \textbf{1)} compared to other LP relaxation based methods, our method directly gives the valid integer label configuration, without any rounding techniques as post-processing; 
\textbf{2)} compared to the exact methods like branch-and-bound \cite{branch-and-bound-1960} and cutting-plane \cite{cutting-plane-1960}, our method optimizes one single continuous problem once, rather than multiple times. 
Experiments on benchmarks from Probabilistic Inference Challenge (PIC 2011) \cite{PIC-2011} and OpenGM 2 \cite{opengm2-ijcv-2015} verify the competitive performance of LS-LP against state-of-the-art MAP inference methods. 

The main contributions of this work are three-fold. 
\textbf{1)} We propose a continuous but equivalent reformulation of the MAP inference problem. 
\textbf{2)} We present the ADMM algorithm for optimizing the perturbed LS-LP problem, which is proved to be globally convergent to the $\epsilon$-KKT point of the original LS-LP problem. The analysis of convergence rate is also presented. 
\textbf{3)} Experiments on benchmark datasets verify the competitive performance of our method compared to state-of-the-art MAP inference methods.

\section{Related Work}
\label{sec:2}

As our method is closely related to LP relaxation based MAP inference methods, here we mainly review MAP inference methods of this category. For other categories of methods, such as message passing and move making, we refer the readers to \cite{variational-inference-book-2008} and \cite{opengm2-ijcv-2015} for more details.
Although some off-the-shelf LP solvers can be used to optimize the LP relaxation problem, in many real-world applications the problem scale is too large to adopt these solvers. Hence, most methods focus on developing efficient algorithms to optimize the dual LP problem.
Block coordinate descent methods \cite{MPLP-NIPS-2007,TRWS-PAMI-2006} are fast, but they may converge to sub-optimal solutions. Sub-gradient based methods \cite{PSDD-ICCV-2007,DD-SG-2012} can converge to global solutions, but their convergence is slow. Their common drawback is the non-smoothness of the dual objective function. To handle this difficulty, some smoothing methods have been developed. The Lagrangian relaxation \cite{Lagrangian-relaxation-MAP-2007} method uses the smooth log-sum-exp function to approximate the non-smooth max function in the dual objective. A proximal regularization \cite{proximal-regularization-ICML-2010} or an $\ell_2$ regularization term \cite{smooth-strong-MAP-2015} is added to the dual objective. Moreover, the steepest $\epsilon$-descent method proposed in \cite{epsilon-descent-nips-2012} and \cite{FW-epsilon-descent-icml-2014} can accelerate the convergence of the standard sub-gradient based methods. Parallel MAP inference methods based on ADMM have also been developed to handle large-scale inference problems. For example, AD3 \cite{AD3-JMLR-2015,AD3-ICML-2011} and Bethe-ADMM \cite{Bethe-ADMM-UAI-2013} optimize the primal LP problem, while ADMM-dual \cite{dual-ADMM-ECML-2011} optimizes the dual LP problem.
The common drawback of these methods is that they are likely to produce fractional solutions, since the underlying problem is merely a relaxation to the MAP inference problem.

Another direction is pursuing tighter relaxations, such as high-order consistency \cite{MAP-LP-thesis-2010} and SDP relaxation \cite{SDP-relaxation-2002}.
But they are often more computationally expensive than LP relaxations.
In contrast, the formulation of the proposed LS-LP is an exact reformulation of the original MAP inference problem, and the adopted ADMM algorithm can explicitly produce valid integer label configurations, without any rounding operation. In comparison with other expensive exact MAP inference methods (\eg, Branch-and-Bound \cite{branch-and-bound-1960} and cutting plane \cite{cutting-plane-1960}), LS-LP is very efficient owing to the resulting parallel inference, similar to other ADMM based methods.

Another related work is $\ell_p$-Box ADMM \cite{WU-lpbox-ADMM-PAMI-2018}, which is a framework to optimize the general integer program. 
The proposed LS-LP is inspired by this framework, where the integer constraints are replaced by the intersection of two continuous constraints.
However, \textbf{1)} LS-LP is specifically designed for MAP inference, as it replaces the valid integer configuration space (\eg, $\{(0,1), (1,0)\}$ for the variable with binary states), rather than the whole binary space (\eg, $\{(0,0), $ $(0,1), (1,0), (1,1)\}$) as did in $\ell_p$-Box ADMM. 
\textbf{2)} LS-LP is tightly combined with LP relaxation, and the ADMM algorithm decomposes the problem into multiple simple sub-problems by utilizing the structure of the factor graph, which allows parallel inference for any type of inference problems (\eg, multiple variable states and high-order factors). In contrast, $\ell_p$-Box ADMM does not assume any special property for the objective function, and it optimizes all variable nodes in one sub-problem. Especially for large-scale models, the sub-problem involved in $\ell_p$-Box ADMM will be very cost. 
\textbf{3)} As LP relaxation is parameterized according to the factor graph, any type of graphical models (\eg, directed models, high-order potentials, asymmetric potentials) can be naturally handled by LS-LP. In contrast, $\ell_p$-Box ADMM needs to transform the inference objective based on MRF models to some simple forms (\eg, binary quadratic program (BQP)). However, the transformation is non-trivial in some cases. For example, if there are high-order potentials, the graphical model is difficult to input into a BQP problem.

\section{Background}
\label{sec: background}

\subsection{Factor Graph} 
\label{sec: subsec factor graph}

Denote $\mathbf{G} = \{\g_1, \g_2, \ldots, \g_N\}$ as a set of $N$ random variables in a discrete space $\mathcal{X} = \mathcal{X}_1 \times \ldots \times \mathcal{X}_N$, where $\mathcal{X}_i = \{0, \ldots, r_i -1\}$ with $r_i = |\mathcal{X}_i|$ being the number of possible states of $\g_i$.
The joint probability of $\mathbf{G}$ is formulated based on a factor graph $G$ \cite{koller-pgm-2009},
\begin{flalign}
P(\mathbf{G}) \propto \exp \big( \sum_{i\in V} \boldsymbol{\theta}_i(\g_i) + \sum_{\alpha \in F} \boldsymbol{\theta}_{\alpha}(\g_{\alpha}) \big),
\label{eq: P(X)}
\end{flalign}
where $G = (V, F, E)$ with $V = \{1, \ldots, N\}$ being the node set of variables, $F$ being the node set of factors, as well as the edge set $E \subseteq V \times F$ linking the variable and factor nodes. A simple MRF model and its factor graph are shown in Fig. \ref{fig: factor graph}(a,b).
We refer the readers to \cite{koller-pgm-2009} for the detailed definition of the factor graph.
$\g_{\alpha}$ indicates the label configuration of the factor $\alpha$, and its state will be determined according to the states of connected variable nodes, \ie, $\{\g_i | i \in \mathcal{N}_{\alpha} \}$, with $\mathcal{N}_{\alpha}$ being the set of neighborhood variable nodes of the factor $\alpha$. 
$\boldsymbol{\theta}_i(\cdot)$ denotes the unary log potential (logPot) function, while $\boldsymbol{\theta}_{\alpha}(\cdot)$ indicates the factor logPot function.

\subsection{MAP Inference as Linear Program}

Given $P(\mathbf{G})$, an important task is to find the most probable label configuration of $\mathbf{G}$, referred to as MAP inference,
\begin{flalign}
\text{MAP}(\boldsymbol{\theta}) = \max_{\mathbf{G} \in \mathcal{X}} \sum_{i\in V} \boldsymbol{\theta}_i(\g_i) + \sum_{\alpha \in F} \boldsymbol{\theta}_{\alpha}(\g_{\alpha}).
\label{eq: MAP inference}
\end{flalign}

Eq. (\ref{eq: MAP inference}) can be reformulated as the integer linear program (ILP) \cite{variational-inference-book-2008}, 
\begin{flalign}
&\text{ILP}(\boldsymbol{\theta})
  = \max_{\boldsymbol{\mu}} \sum_{i\in V} \boldsymbol{\theta}_i^\top \boldsymbol{\mu}_i +
 \sum_{\alpha \in F} \boldsymbol{\theta}_{\alpha}^\top \boldsymbol{\mu}_{\alpha}
 =
  \max_{\boldsymbol{\mu}}  \langle \boldsymbol{\theta}, \boldsymbol{\mu} \rangle,
  \label{eq: MAP inference ILP}
 \\
 & \text{s.t.} ~  \boldsymbol{\mu} \in \mathcal{L}_G \cap \{0,1\}^{|\boldsymbol{\mu}|},
   \nonumber
\end{flalign}
where $\boldsymbol{\theta} = (\ldots;\boldsymbol{\theta}_i; \ldots; \boldsymbol{\theta}_{\alpha}; \ldots), i\in V, \alpha \in F $ denotes the log potential (logPot) vector, derived from $\boldsymbol{\theta}_i(\g_i)$ and $\boldsymbol{\theta}_{\alpha}(\g_{\alpha})$.
$\boldsymbol{\mu} = [\boldsymbol{\mu}_V; \boldsymbol{\mu}_F]$, where $\boldsymbol{\mu}_V = [\boldsymbol{\mu}_1; \ldots; \boldsymbol{\mu}_{|V|}]$ and $\boldsymbol{\mu}_F = [\boldsymbol{\mu}_1; \ldots; \boldsymbol{\mu}_{|F|}]$.
$\boldsymbol{\mu}_i \in \{0,1\}^{|\mathcal{X}_i|}$ indicates the label vector corresponding to $\g_i$: if the state of $\g_i$ is $t$, then $\boldsymbol{\mu}_i(t)=1$, while all other entries are $0$.
Similarly, $\boldsymbol{\mu}_{\alpha} \in \{0,1\}^{|\mathcal{X}_{\alpha}|}$ indicates the label vector corresponding to $\g_{\alpha}$.
The local marginal polytope is defined as follows,
\begin{flalign}
\mathcal{L}_G = \big\{ \boldsymbol{\mu} \big| & \boldsymbol{\mu}_{\alpha} \in \Delta^{|\boldsymbol{\mu}_{\alpha}|}, \forall \alpha \in F; 
\\
&\boldsymbol{\mu}_i  = \M_{i\alpha} \boldsymbol{\mu}_{\alpha}, ~\forall (i, \alpha) \in E \big\}.
\nonumber
\end{flalign}
with $\Delta^{|\a|} = \{\a | \boldsymbol{1}^\top \a = 1, \a \geq \boldsymbol{0} \}$ being the probability simplex, 
and the second constraint ensures the local consistency between $\boldsymbol{\mu}_i$ and $\boldsymbol{\mu}_{\alpha}$. 
$\M_{i\alpha} \in \{0, 1\}^{|\mathcal{X}_i| \times |\mathcal{X}_{\alpha}|}$ of the local consistency constraint included in $\mathcal{L}_G$ is defined as: 
the entry of $\M_{i\alpha}$ is $1$ if $\g_{\alpha} \sim \g_i$, where $\g_{\alpha} \sim \g_i$ indicates the state of $\g_i$ and the state of the corresponding element in $\g_{\alpha}$ are the same; otherwise, the entry is $0$. 
For example, we consider a binary-state variable node $\boldsymbol{\mu}_i \in \{0,1\}^2$ and a pairwise factor node $\boldsymbol{\mu}_{\alpha} \in \{0,1\}^4$ connected to two variable nodes (the variable node $i$ is the first). 
The first entry of $\boldsymbol{\mu}_i$ indicates the score of choosing state $0$, while the second entry corresponds to that of choosing state $1$.
The four entries of $\boldsymbol{\mu}_{\alpha}$ indicate the scores of four label configurations of two connected variables, \ie, $(0,0), (0,1), (1,0), (1,1)$. 
In this case, $\M_{i \alpha} = [1, 1, 0, 0; 0, 0, 1, 1]$.

\begin{figure}[t]
\centering
\includegraphics[width=0.48\textwidth, height=1.35in]{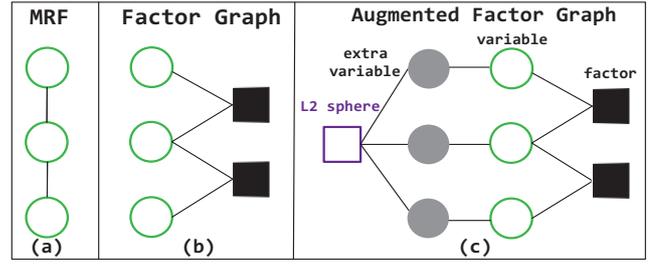}
\caption{ An example of (a) MRF, (b) factor graph corresponding to LP and (c) augmented factor graph corresponding to LS-LP. }
\label{fig: factor graph}
\vspace{-1em}
\end{figure}

Moreover, Eq. (\ref{eq: MAP inference}) can also be rewritten as 
\begin{flalign}
\text{MAP}(\boldsymbol{\theta})
 = \max_{\boldsymbol{\mu} \in \mathcal{M}_G}  \langle \boldsymbol{\theta}, \boldsymbol{\mu} \rangle, 
\label{eq: MAP inference with M_G}
\end{flalign}
where the marginal polytope is defined as follows,
\begin{flalign}
\mathcal{M}_G = \{ \boldsymbol{\mu} ~|~ \exists P(X), ~ \text{such that} ~ \boldsymbol{\mu}_i, \boldsymbol{\mu}_{\alpha} \in \mathcal{L}_G \}.
\end{flalign} 
Solving $\text{MAP}(\boldsymbol{\theta})$ is difficult (NP-hard in general), especially for large scale problems.
Instead, the approximation over $\mathcal{L}_G$ is widely adopted, as follows:
\begin{flalign}
\text{LP}(\boldsymbol{\theta})
 & =  \max_{\boldsymbol{\mu} \in \mathcal{L}_G}  \langle \boldsymbol{\theta}, \boldsymbol{\mu} \rangle
  \geq \text{ILP}(\boldsymbol{\theta}) = \text{MAP}(\boldsymbol{\theta}),
\label{eq: MAP inference over L(G)}
\end{flalign}
which is called LP relaxation. 
Note that here $\boldsymbol{\mu}_i$ and $\boldsymbol{\mu}_{\alpha}$ are continuous variables, and they are considered as local marginals of $\g_i$ and $\g_{\alpha}$, respectively.

According to \cite{variational-inference-book-2008}, the characteristics of $\text{LP}(\boldsymbol{\theta})$, $\text{MAP}(\boldsymbol{\theta})$ and their relationships are briefly summarized in Lemma \ref{lemma1: L(G) and M(G)}.

\begin{lemma} \cite{variational-inference-book-2008}
The relationship between $\mathcal{M}_G$ and $\mathcal{L}_G$, and that  between $\text{MAP}(\boldsymbol{\theta})$ and $\text{LP}(\boldsymbol{\theta})$ are as follows.
\begin{itemize}
\item $\mathcal{M}_G \subseteq \mathcal{L}_G$;
\item $\text{MAP}(\boldsymbol{\theta}) \leq \text{LP}(\boldsymbol{\theta})$;
\item All vertices of $\mathcal{M}_G$ are integer, while  $\mathcal{L}_G$ includes both integer and fractional vertices. And the set of integer vertices of  $\mathcal{L}_G$ is same with the set of the vertices of $\mathcal{M}_G$. All non-vertices in $\mathcal{M}_G$ and $\mathcal{L}_G$ are fractional points.
\item Since both $\mathcal{M}_G$ and $\mathcal{L}_G$ are convex polytopes, the global solutions of $\text{MAP}(\boldsymbol{\theta})$ and $\text{LP}(\boldsymbol{\theta})$ will be on the vertices of $\mathcal{M}_G$ and $\mathcal{L}_G$, respectively.
\item The global solution $\boldsymbol{\mu}^*$ of $\text{LP}(\boldsymbol{\theta})$ can be fractional or integer. If it is integer, then it is also the global solution of $\text{MAP}(\boldsymbol{\theta})$.
\end{itemize}
\label{lemma1: L(G) and M(G)}
\end{lemma}

\subsection{Kurdyka-Lojasiewicz Inequality}
\label{sec: subsec KL property}

The Kurdyka-Lojasiewicz inequality was firstly proposed in \cite{lojasiewicz1963propriete}, and it has been widely used in many recent works \cite{attouch2010proximal,wotao-yin-arxiv-2015,admm-nonconvex-siam-2015} for the convergence analysis of non-convex problems. Since it will also be used in the later convergence analysis of our algorithm, it is firstly produced here, as shown in Definition \ref{definition: KL property}.

\begin{definition}\cite{attouch2010proximal}
A function $f : \mathbb{R}^n \rightarrow \mathbb{R} \cup {+\infty}$ is said to have the Kurdyka-Lojasiewicz (KL) property at $x^* \in \text{dom} (\partial f)$ ($\text{dom}(\cdot)$ denotes the domain of function, $\partial$ indicates the sub-gradient operator), if the following two conditions hold
\begin{itemize}
    \item there exist a constant $\eta \in (0, +\infty]$, a neighborhood $\mathcal{V}_{x^*}$ of $x^*$, as well as a continuous concave function $\varphi: [0, \eta) \rightarrow \mathbb{R}_{+}$, with $\varphi(0)=0$ and $\varphi$ is differentiable on $(0,\eta)$ with positive derivatives.
    \item $\forall x \in \mathcal{V}_{x^*}$ satisfying $f(x^*) < f (x) < f(x^*) + \eta$, the Kurdyka-Lojasiewicz inequality holds
    \begin{flalign}
     \varphi'(f(x) - f(x^*)) \text{dist}(0, \partial f(x)) \geq 1.
    \end{flalign}
\end{itemize}
\label{definition: KL property}
\end{definition}

\noindent
{\bf Remark.} 
According to \cite{attouch2010proximal,bolte2007lojasiewicz,bolte2007clarke}, if $f$ is semi-algebraic, then it satisfies the KL property with $\varphi(s) = c s^{1-p}$, where $p \in [0,1)$ and $c>0$ are constants. This point will be used in later analysis of convergence.

\section{MAP Inference via $\ell_2$-sphere Linear Program Reformulation}

\subsection{Equivalent Reformulation}

Firstly, we introduce the $\ell_2$-sphere constraint \cite{WU-lpbox-ADMM-PAMI-2018},
\begin{flalign}
 \mathcal{S} = \big\{\x \in \mathbb{R}^n \mid ~ \parallel \x - \frac{1}{2} \boldsymbol{1}\parallel_2^2 = \frac{n}{4} \big\}.
 \label{eq: S}
\end{flalign}
Note that $\mathcal{S}$ is defined with respect to the vector $\x$, rather than individual scalars $x_i, i=1,\ldots,n$. 
We propose to add the $\ell_2$-sphere constraint onto the variable nodes $\boldsymbol{\mu}_V$. Combining this with LP relaxation (see Eq. (\ref{eq: MAP inference over L(G)}) ), we propose a new formulation for MAP inference,
\begin{flalign}
\text{LS-LP}(\boldsymbol{\theta})
 & =  \max_{\boldsymbol{\mu}}  \langle \boldsymbol{\theta}, \boldsymbol{\mu} \rangle , 
 ~~
 \text{s.t.} ~  \boldsymbol{\mu} \in \mathcal{L}_G,
 \boldsymbol{\mu}_V \in \mathcal{S}.
  \label{eq: MAP inference over L(G) and S}
\end{flalign}
Due to the non-convex constraint $\mathcal{S}$, it is no longer a linear program. However, to emphasize its relationship to LP relaxation, we still denote it as a $\ell_2$-sphere constrained linear program (LS-LP) reformulation. More importantly, as shown in Proposition \ref{proposition1: equivalence of LP-LS and MAP}, LS-LP is equivalent to the original MAP inference problem, rather than a relaxation as in LP. 
Inspired by the constraint separation in $\ell_p$-Box ADMM \cite{WU-lpbox-ADMM-PAMI-2018}, we introduce the extra variable $\boldsymbol{\upsilon}$ to reformulate (\ref{eq: MAP inference over L(G) and S}) as 
\begin{flalign}
& \text{LS-LP}(\boldsymbol{\theta})
  =  \max_{\boldsymbol{\mu}, \boldsymbol{\upsilon}}  \langle \boldsymbol{\theta}, \boldsymbol{\mu} \rangle = \min_{\boldsymbol{\mu}, \boldsymbol{\upsilon}}  \langle -\boldsymbol{\theta}, \boldsymbol{\mu} \rangle,
  \label{eq: MAP inference over L(G) and v over S}
\\
& \text{s.t.} ~   \boldsymbol{\mu} \in \mathcal{L}_G,
 \boldsymbol{\upsilon} \in \mathcal{S}, \boldsymbol{\mu}_{i} = \boldsymbol{\upsilon}_{i}, \forall i \in V,
\nonumber
\end{flalign}
where $\boldsymbol{\upsilon} = [\boldsymbol{\upsilon}_1; \ldots; \boldsymbol{\upsilon}_{i}; \ldots; \boldsymbol{\upsilon}_{|V|}], i \in V$ is the concatenated vector of all extra variable nodes.
The combination of the original factor graph and these extra variable nodes is referred to as {\it augmented factor graph} (AFG). 
An example of AFG corresponding to Problem (\ref{eq: MAP inference over L(G) and v over S}) is shown in Figure \ref{fig: factor graph}(c). The gray circles correspond to extra variables $\boldsymbol{\upsilon}$, and connections to the purple box indicate that $\boldsymbol{\upsilon} \in \mathcal{S}$.
Note that AFG does not satisfy the definition of the standard factor graph, where connections only exist between variables nodes and factor nodes. However, AFG provides a clear picture of the structure of LS-LP and the node relationships.
The proposed LS-LP problem is equivalent to the original MAP inference problem, as shown in Proposition \ref{proposition1: equivalence of LP-LS and MAP}. 
It means that the global solutions of this two problems are equivalent.

\begin{lemma}
 The following constraint spaces are equivalent,
\begin{flalign}
 & ~ \mathcal{C}_1 = \{ \boldsymbol{\mu} \mid \boldsymbol{\mu} \in \mathcal{L}_G \cap \{0,1\}^{|\boldsymbol{\mu}|} \}
\nonumber
 \\
  \equiv & ~
 \mathcal{C}_2 = \{ \boldsymbol{\mu} \mid \boldsymbol{\mu} \in \mathcal{L}_G ~\text{and}~ \boldsymbol{\mu}_V \in \mathcal{S} \}
 \nonumber
\\
\equiv &~
\mathcal{C}_3 = \{ \boldsymbol{\mu} \mid \boldsymbol{\mu} \in \mathcal{M}_G \cap \{0,1\}^{|\boldsymbol{\mu}|} \}.
\label{eq: equivalence in Proposition 1}
\end{flalign}
\label{lemma2: equivalence of different constraint spaces}
\end{lemma}

\begin{proof}
We start from $\mathcal{C}_2$, where we have
\begin{flalign}
\boldsymbol{\mu}_V \in \mathcal{S} \Longleftrightarrow \sum_{i \in V} \parallel \boldsymbol{\mu}_{i} - \frac{1}{2} \parallel _2^2 = \frac{\sum_{i \in V} |\mathcal{X}_{i}|}{4}.
\label{eq: constraint S to detailed formulation}
\end{flalign}
Besides, the following relations hold
\begin{flalign}
&\boldsymbol{\mu} \in \mathcal{L}_G 
\Longleftrightarrow 
\boldsymbol{\mu}_{\alpha} \in \Delta^{|\mathcal{X}_{\alpha}|} ~\text{and}~ \boldsymbol{\mu}_i = \M_{i \alpha} \boldsymbol{\mu}_{\alpha} 
\\
\Rightarrow &
\boldsymbol{\mu}_i \in [0,1]^{|\mathcal{X}_i|}
\Rightarrow 
\parallel \boldsymbol{\mu}_{i} - \frac{1}{2} \parallel _2^2 \leq
\frac{ |\mathcal{X}_{i}|}{4}, 
\nonumber
\end{flalign}
$\forall i \in V, \forall (i, \alpha) \in E$. 
The equation in the last relation holds if and only if $\boldsymbol{\mu}_{i} \in \{0, 1\}^{|\mathcal{X}_i|}$. 
Combining with (\ref{eq: constraint S to detailed formulation}), we conclude that 
$\boldsymbol{\mu}_{i} \in \{0, 1\}$ holds $\forall i \in V$. 
Consequently, utilizing the local consistency constraint 
$\boldsymbol{\mu}_i = \M_{i \alpha} \boldsymbol{\mu}_{\alpha}$, we obtain that 
$\boldsymbol{\mu}_{\alpha} \in \{0, 1\}$ also holds $\forall \alpha \in F$. 
Thus, we have $\boldsymbol{\mu} \in \{0,1\}^{|\boldsymbol{\mu}|}$. 
Then, the relation $\mathcal{C}_1 \equiv \mathcal{C}_2$ is proved.

Besides, as shown in Lemma \ref{lemma1: L(G) and M(G)}, the set of integer vertices of  $\mathcal{L}_G$ is same with the one of $\mathcal{M}_G$, and all non-vertices in $\mathcal{M}_G$ and $\mathcal{L}_G$ are fractional points. 
Thus, it is easy to know $\mathcal{C}_1 \equiv \mathcal{C}_3$. 
Hence the proof is finished.
\end{proof}

\begin{theorem}
Utilizing Lemma \ref{lemma2: equivalence of different constraint spaces}, the aforementioned MAP inference problems have the following relationships,
\begin{flalign}
\text{LS-LP}(\boldsymbol{\theta}) = \text{ILP}(\boldsymbol{\theta}) = \text{MAP}(\boldsymbol{\theta}) \leq \text{LP}(\boldsymbol{\theta}).
\label{eq: equivalence of LP-LS and MAP}
\end{flalign}
\label{proposition1: equivalence of LP-LS and MAP}
\end{theorem}

\begin{proof}
According to Lemma \ref{lemma1: L(G) and M(G)}.3 and \ref{lemma1: L(G) and M(G)}.4, as well as $\mathcal{C}_2 \equiv \mathcal{C}_3$ in Lemma \ref{lemma2: equivalence of different constraint spaces} (see Eq. (\ref{eq: equivalence in Proposition 1})), we have
\begin{flalign}
&\max_{\boldsymbol{\mu} \in \mathcal{M}_G} \langle \boldsymbol{\theta}, \boldsymbol{\mu} \rangle
=
\max_{\boldsymbol{\mu} \in \mathcal{C}_3 } \langle \boldsymbol{\theta}, \boldsymbol{\mu} \rangle 
= 
\max_{\boldsymbol{\mu} \in \mathcal{C}_2 } \langle \boldsymbol{\theta}, \boldsymbol{\mu} \rangle
\\
&  \Longleftrightarrow 
  \text{MAP}(\boldsymbol{\theta}) = \text{LS-LP}(\boldsymbol{\theta}).
\end{flalign}
Combining with $\text{MAP}(\boldsymbol{\theta}) = \text{ILP}(\boldsymbol{\theta}) \leq \text{LP}(\boldsymbol{\theta})$ (see Eq. (\ref{eq: MAP inference over L(G)})), the proof is finished.
\end{proof}

\subsection{A General Form and KKT Conditions}

For clarity, we firstly simplify the notations and formulations in Eq. (\ref{eq: MAP inference over L(G) and v over S}) to the general shape, 
\begin{flalign}
& \text{LS-LP}(\boldsymbol{\theta})
  =  
  \min_{\x, \y} f(\x) + h(\y), ~\text{s.t.} ~  \A \x = \B \y.
  \label{eq: LS-LP(x,y)}
\end{flalign}
Our illustration for (\ref{eq: LS-LP(x,y)}) is separated into three parts, as follows:
\begin{enumerate}
    \item {\bf Variables}. $\x = [ \boldsymbol{\mu}_1; \ldots; \boldsymbol{\mu}_{|V|} ] \in \mathbb{R}^{\sum_{i}^{ V} |\mathcal{X}_i|}$, and it concatenates all variable nodes $\boldsymbol{\mu}_V$.  
    $\y = [\y_1; \ldots; \y_{|V|}]$ with $\y_i = [\boldsymbol{\upsilon}_i; \boldsymbol{\mu}_{\alpha_{i,1}}; \ldots; \boldsymbol{\mu}_{\alpha_{i,|\mathcal{N}_i|}}] \in \mathbb{R}^{|\mathcal{X}_i| + \sum_{\alpha}^{\mathcal{N}_i} |\mathcal{X}_{\alpha}|}$.
    $\y$ concatenates all factor nodes $\boldsymbol{\mu}_V$ and the extra variable nodes $\boldsymbol{\upsilon}$; $\y_i$ concatenates the factor nodes and the extra variable node  connected to the $i$-th variable node $\boldsymbol{\mu}_i$. 
    $\mathcal{N}_i$ indicates the set of neighborhood factor nodes connected to the $i$-th variable node; 
    the subscript $\alpha_{i,j}$ indicates the $j$-th factor connected to the $i$-th variable, with $i \in V$ and $j \in \mathcal{N}_i$. 
    \item {\bf Objective functions}. 
    $f(\x)$ $ = \w_{\x}^\top \x$ with $\w_{\x} = - [\boldsymbol{\theta}_1; \ldots; $ $\boldsymbol{\theta}_{|V|}]$. $h(\y) = g(\y) + \w_{\y}^\top \y$, with $\w_{\y} = [\w_1; \ldots;$ $\w_{|V|}]$ with $\w_{i} = -[\boldsymbol{0}; \frac{1}{|\mathcal{N}_{\alpha_{i,1}}|} \boldsymbol{\theta}_{\alpha_{i,1}};$ 
    $\ldots; \frac{1}{|\mathcal{N}_{\alpha_{i,|\mathcal{N}_i|}}|} \boldsymbol{\theta}_{\alpha_{i,|\mathcal{N}_i|}}]$, and $\mathcal{N}_{\alpha} = \{ i \mid (i, \alpha) \in E\}$ being the set of neighborhood variable nodes connected to the $\alpha$-th factor. 
    $g(\y) = \mathbb{I}(\boldsymbol{\upsilon} \in \mathcal{S}) + \sum_{\alpha \in F} \mathbb{I}(\boldsymbol{\mu}_{\alpha} \in \Delta^{|\mathcal{X}_{\alpha}|})$, with $\mathbb{I}(a)$ being the indicator function: $\mathbb{I}(a)=0$ if $a$ is true, otherwise $\mathbb{I}(a)=\infty$.
    \item {\bf Constraint matrices}. 
    The constraint matrix $\A = \text{diag}($ $\A_1, \ldots, \A_i, \ldots, \A_{|V|})$ with $\A_i = [\mathbf{I}_{|\mathcal{X}_i|}; \ldots; \mathbf{I}_{|\mathcal{X}_i|} ] \in \{0,1\}^{(|\mathcal{N}_i| +1)|\mathcal{X}_i| \times  |\mathcal{X}_i|}$. 
    $\B = \text{diag}(\B_1, \ldots,$ $ \B_i, \ldots, \B_{|V|})$, with $\B_i = \text{diag}(\mathbf{I}_{|\mathcal{X}_i|}, \M_{i, \alpha_{i,1}}, \ldots, \M_{i, \alpha_{i, |\mathcal{N}_i|}} )$. 
    $\A$ summarizes all constraints on $\boldsymbol{\mu}_V$, while $\B$ collects all constraints on $\boldsymbol{\mu}_F$ and $\boldsymbol{\upsilon}$. 
\end{enumerate}
Note that Problem (\ref{eq: LS-LP(x,y)}) has a clear structure with two groups of variables, corresponding the augmented factor graph (see Fig. \ref{fig: factor graph}(c)).

\begin{definition} \label{definition: epsilon kkt solution}
The solution $(\x^*, \y^*)$ of the LS-LP problem (\ref{eq: LS-LP(x,y)}) is said to be the KKT point if the following conditions are satisfied:
\begin{flalign}
 \B^\top \boldsymbol{\lambda}^* \in \partial h(\y^*), ~
 \nabla f(\x^{*}) = - \A^\top \boldsymbol{\lambda}^*, ~
 \A \x^* = \B \y^*, \hspace{-1em}
 \label{eq: feasible and stationary condition of original LS-LP}
\end{flalign}
where $\boldsymbol{\lambda}^*$ denotes the Lagrangian multiplier; $\partial h$ indicates the sub-gradient of $h$, while $\nabla f$ represents the gradient of $f$. 
Moreover, $(\x^*,$ $\y^*)$ is considered as the $\epsilon$-KKT point if the following conditions hold:
\begin{flalign}
 & \text{dist}(\B^\top \boldsymbol{\lambda}^*, \partial h(\y^*)) \leq O(\epsilon),
 ~ 
 \| \nabla f(\x^{*}) + \A^\top \boldsymbol{\lambda}^* \| \leq O(\epsilon), 
 \nonumber 
 \\
 & \| \A \x^* - \B \y^* \| \leq O(\epsilon).
 \label{eq: epsilon feasible and stationary condition of original LS-LP}
\end{flalign}
\end{definition}

\begin{algorithm}[!t]
\small
\caption{The perturbed ADMM algorithm}
\label{alg: perturbed admm}
\begin{algorithmic}[1]
\Require The initialization $\y^0, \hat{\x}^0, \boldsymbol{\lambda}^0$, the perturbation $\epsilon$, the hyper-parameter $\rho$
\For {$k = 0$ to $K$}:
\State Update $\y^{k+1}$ as follows (see Section \ref{sec: subsec update y in perturbed admm} for details)
\begin{flalign}
\hspace{2.5em} \y^{k+1} =  \argmin_{\y} \mathcal{L}_{\rho,\epsilon}(\y, \hat{\x}^k, \boldsymbol{\lambda}^k)
\label{eq: update y in perturbed admm}
\end{flalign}
\State Update $\hat{\x}^{k+1}$ as follows (Section \ref{sec: subsec update x in perturbed admm} for details)
\begin{flalign}
\hspace{2.5em} \hat{\x}^{k+1} =  \argmin_{\hat{\x}} \mathcal{L}_{\rho,\epsilon}(\y^{k+1}, \hat{\x}, \boldsymbol{\lambda}^k)
\label{eq: update x in perturbed admm}
\end{flalign}
\State Update $\boldsymbol{\lambda}^{k+1}$ (see Section \ref{sec: subsec update lambda in perturbed admm} for details)
\begin{flalign}
 \hspace{2.5em} \boldsymbol{\lambda}^{k+1} =  \boldsymbol{\lambda}^k + \rho (\hat{\A} \hat{\x}^{k+1} - \B \y^{k+1} )
  \label{eq: update lambda in perturbed admm}
\end{flalign}
\State Check stopping criterion, as shown in Section \ref{sec: subsec complexity and implementation}
\EndFor
\Return $\y^*, \x^*$
\end{algorithmic}
\end{algorithm}

\section{Perturbed ADMM Algorithm for LS-LP}
\label{sec: perturbed ADMM for LS-LP}

We propose a perturbed ADMM algorithm to optimize the following perturbed augmented Lagrangian function, 
\begin{flalign}
 \mathcal{L}_{\rho,\epsilon}(\y, \hat{\x}, \boldsymbol{\lambda})
 = &  
  \hat{f}(\hat{\x}) + h(\y) +  \boldsymbol{\lambda}^\top ( \hat{\A} \hat{\x} - \B \y ) 
  \nonumber
  \\
  & + \frac{\rho}{2} \| \hat{\A} \hat{\x} - \B \y \|_2^2,
  \label{eq: augmented L of LS-LP(x,y) with perturbation}
\end{flalign}
where $\hat{\A} = [\A, \epsilon \mathbf{I}]$ with a sufficiently small constant $\epsilon > 0$, then $\hat{\A}$ is full row rank. 
$\hat{\x} = [\x; \bar{\x}]$, with $\bar{\x} = [\bar{\x}_1; \ldots; \bar{\x}_{|V|}] \in \mathbb{R}^{\sum_i^{V} (|\mathcal{N}_i|+1) |\mathcal{X}_i|}$ and  $\bar{\x}_i = [\boldsymbol{\mu}_i; \ldots; \boldsymbol{\mu}_i] \in \mathbb{R}^{(|\mathcal{N}_i|+1) |\mathcal{X}_i|}$. 
$\hat{f}(\hat{\x}) = f(\x) + \frac{1}{2}\epsilon \hat{\x}^\top \hat{\x}$. 
Note that both $\hat{\A}$ and $\B$ are full row rank, and the second-order gradient $\nabla^2 \hat{f}(\hat{\x}) = \epsilon \I$ is bounded. These properties will play key roles in our later analysis of convergence.

Following the conventional ADMM algorithm, the solution to the LS-LP problem (\ref{eq: LS-LP(x,y)}) can be obtained through optimizing the following sub-problems based on (\ref{eq: augmented L of LS-LP(x,y) with perturbation}) iteratively. 
The general structure of the algorithm is summarized in Algorithm \ref{alg: perturbed admm}.

\subsection{Sub-Problem w.r.t. $\y$ in LS-LP Problem}
\label{sec: subsec update y in perturbed admm}

Given $\hat{\x}^k$ and $\boldsymbol{\lambda}^k$, 
$\y^{k+1}$ can be updated by solving the sub-problem (\ref{eq: update y in perturbed admm}) (see Algorithm \ref{alg: perturbed admm}). 
According to the definitions of $\hat{\A}, \hat{\x}, \B, \y$, this problem can be further separated to the following two independent sub-problems, which can be solved in parallel. 

\vspace{1em}
\noindent
{\bf Update $\bupsilon^{k+1}$}:
\begin{flalign}
 \min_{\bupsilon \in \mathcal{S}} \sum_{i \in V} \big[ - (\boldsymbol{\lambda}_{i}^k)^\top  \boldsymbol{\upsilon}_{i}
+ \frac{\rho_{i}}{2} \parallel (1+\epsilon) \boldsymbol{\mu}_i^k - \boldsymbol{\upsilon}_{i} \parallel_2^2
\big].
\label{eq: subproblem of v in perturbed admm}
\end{flalign}
It has a closed form solution as follows
\begin{flalign}
\boldsymbol{\upsilon}^{k+1}  = \mathcal{P}_{\mathcal{S}}(\overline{ \boldsymbol{\upsilon}}^{k+1}),
\end{flalign}
where $\overline{ \boldsymbol{\upsilon}}^{k+1} = [ \boldsymbol{\upsilon}_{1}^{k+1}; \ldots;  \boldsymbol{\upsilon}_{|V|}^{k+1} ]$ 
with 
$\boldsymbol{\upsilon}_{i}^{k+1} = (1+\epsilon) \boldsymbol{\mu}_i^k + \frac{1}{\rho_i} \boldsymbol{\lambda}_{i}^k$.
$\mathcal{P}_{\mathcal{S}}(\cdot)$ is the projection onto $\mathcal{S}$: 
$\mathcal{P}_{\mathcal{S}}(\a) = \frac{n^{1/2}}{2} \times \frac{\overline{\a}}{\parallel \overline{\a}\parallel_2} + \frac{1}{2} \boldsymbol{1}_n$, with $\overline{\a} = \a - \frac{1}{2}  \boldsymbol{1}_n$ and $n$ being the dimension of $\a$.
As demonstrated in \cite{WU-lpbox-ADMM-PAMI-2018}, this projected solution is the optimal solution to (\ref{eq: subproblem of v in perturbed admm}).

\vspace{1em}
\noindent
{\bf Update $\boldsymbol{\mu}_{\alpha}^{k+1}$}: 
The sub-problems w.r.t. $\{ \boldsymbol{\mu}_{\alpha} \}_{ \alpha \in F}$ can be run in parallel $\forall \alpha \in F$, 
\begin{flalign}
 \min_{\boldsymbol{\mu}_{\alpha} \in \Delta^{|\mathcal{X}_{\alpha}|} }  ~  &
  -\boldsymbol{\theta}_{\alpha}^\top \boldsymbol{\mu}_{\alpha}
  + 
   \sum_{i \in \mathcal{N}_{\alpha}} \big[ \frac{\rho_{i \alpha}}{2} \parallel (1+\epsilon) \boldsymbol{\mu}_i^k - \M_{i\alpha} \boldsymbol{\mu}_{\alpha} \parallel_2^2 
\nonumber 
\\
 & - (\boldsymbol{\lambda}_{i \alpha}^{k})^\top \M_{i\alpha} \boldsymbol{\mu}_{\alpha}  \big].
   \label{eq: subproblem of factor simplified}
\end{flalign}
It is easy to know that Problem (\ref{eq: subproblem of factor simplified}) is convex, as $\M_{i\alpha}^\top \M_{i\alpha}$ is positive semi-definite and $\Delta^{|\mathcal{X}_{\alpha}|}$ is a convex set. 
Any off-the-shelf QP solver can be adopted to solve  (\ref{eq: subproblem of factor simplified}). 
In experiments, we adopt the active-set algorithm implemented by a publicly-available toolbox called Quadratic Programming in C (QPC)\footnote{\scriptsize  http://sigpromu.org/quadprog/download.php?sid=3wtwk5tb}, which is written in C and can be called from MATLAB.

\subsection{Sub-Problem w.r.t. $\hat{\x}$ in LS-LP Problem}
\label{sec: subsec update x in perturbed admm}

Given $\y^{k+1}$ and $\boldsymbol{\lambda}^k$, 
$\hat{\x}^{k+1}$ can be updated by solving the sub-problem (\ref{eq: update x in perturbed admm}) (see Algorithm \ref{alg: perturbed admm}).
According to the definition of $\hat{\x}$, this problem can be separated to $|V|$ independent sub-problems w.r.t. $\{\boldsymbol{\mu}_i\}_{i \in V}$, as follows:  
\begin{flalign}
 & \min_{\boldsymbol{\mu}_i } ~
 ( \boldsymbol{\lambda}_{i}^k - \boldsymbol{\theta}_i )^\top \boldsymbol{\mu}_{i} + \frac{\epsilon (|\mathcal{N}_i|+2)}{2} \boldsymbol{\mu}_{i}^\top \boldsymbol{\mu}_{i} +
  \sum_{\alpha \in \mathcal{N}_i} \bigg[
 \\
 &  
  (1+\epsilon) (\boldsymbol{\lambda}_{i \alpha}^{k})^\top  \boldsymbol{\mu}_i  + \frac{\rho_{i \alpha}}{2} \parallel  (1+\epsilon) \boldsymbol{\mu}_i - \M_{i\alpha} \boldsymbol{\mu}_{\alpha}^{k+1} \parallel_2^2  \bigg] 
 \nonumber 
 \\
 &  + \frac{\rho_{i}}{2} \parallel (1+\epsilon) \boldsymbol{\mu}_i - 
 \boldsymbol{\upsilon}_{i} \parallel_2^2 
 \nonumber
 \\
 & 
 = 
 (1+\epsilon) \bigg[ \sum_{\alpha \in \mathcal{N}_i} \big( \boldsymbol{\lambda}_{i \alpha}^{k} - \rho_{i \alpha} \M_{i\alpha} \boldsymbol{\mu}_{\alpha}^{k+1} \big) -  
 \rho_i \boldsymbol{\upsilon}_i^{k+1} - \boldsymbol{\theta}_i 
 \nonumber 
 \\
 & + \boldsymbol{\lambda}_{i}^k
 \bigg]^\top 
 \boldsymbol{\mu}_i + \boldsymbol{\mu}_i^\top \mathbf{Q} \boldsymbol{\mu}_i + \text{const},
 \label{eq: subproblem to u_i}
\end{flalign}
where $\mathbf{Q} = \frac{1}{2}[\epsilon (|\mathcal{N}_i|+2) + \rho_i (1+\epsilon)^2  + \sum_{\alpha \in \mathcal{N}_i} \rho_{i \alpha} (1+\epsilon)^2] \cdot \mathbf{I}$.
The above sub-problem can be further simplified to $\min_{\boldsymbol{\mu}_i} a \cdot \boldsymbol{\mu}_i^\top \boldsymbol{\mu}_i + \mathbf{b}^\top \boldsymbol{\mu}_i$, where $a$ and $\b$ can be easily derived from the above equation. 
Its close-form solution is obtained by setting its gradient to $\mathbf{0}$, \ie, $\boldsymbol{\mu}_i^{k+1} = \frac{\mathbf{b}}{2 a}$.

\subsection{Update $\boldsymbol{\lambda}$ in LS-LP Problem}
\label{sec: subsec update lambda in perturbed admm}
Given $\y^{k+1}$ and $\hat{\x}^{k+1}$, $\boldsymbol{\lambda}^{k+1}$ is updated using  (\ref{eq: update lambda in perturbed admm}) (see Algorithm \ref{alg: perturbed admm}).
Similarly, it can be separately to $|V| + |E|$ independent sub-problems, as follows
\begin{flalign}
 \boldsymbol{\lambda}_{i}^{k+1} & =  \boldsymbol{\lambda}_{i}^{k} + \rho_{i} [ (1+\epsilon) \boldsymbol{\mu}_{i}^{k+1} - \boldsymbol{\upsilon}_{i}^{k+1}], 
\\
 \boldsymbol{\lambda}_{i \alpha}^{k+1} & =  \boldsymbol{\lambda}_{i \alpha}^{k} +  \rho_{i \alpha} [ (1+\epsilon) \boldsymbol{\mu}_{i}^{k+1} - \M_{i\alpha} \boldsymbol{\mu}_{\alpha}^{k+1} ],
 \label{eq: update of dual in perturbed admm}
\end{flalign}
where $i \in V, (i, \alpha) \in E$.

\subsection{Complexity and Implementation Details}
\label{sec: subsec complexity and implementation}

\vspace{0.1em}
\noindent
{\bf Complexity}. ~ 
In terms of computational complexity, as all other update steps have simple closed-form solutions, the main computational cost lies in updating $\boldsymbol{\mu}_{\alpha}$, which is convex quadratic programming with the probability simplex constraint. 
Its computational complexity is $O(|\mathcal{X}_{\alpha}|^3)$.
As the matrix with the largest size is $\mathbf{M}_{i \alpha}^\top \mathbf{M}_{i \alpha} \in \mathbb{R}^{|\mathcal{X}_{\alpha}| \times |\mathcal{X}_{\alpha}|}$ in LS-LP, the space complexity is $\mathcal{O}(\sum_{\alpha \in F} (|\mathcal{X}_{\alpha}|)^2)$. 
Both the computational and space complexity of AD3 are similar with LS-LP. 
More detailed analysis about the computational complexity will be presented 
in Section \ref{sec: runtime}.

\vspace{0.3em}
\noindent
{\bf Implementation details}. ~ 
In each iteration, we use the same value of $\rho$ for all $\rho_i$ and $\rho_{i \alpha}$. After each iteration, we update $\rho$ using an incremental rate $\eta$, \ie, $\rho \leftarrow \rho \times \eta$. A upper limit $\rho_{upper}$ of $\rho$ is also set: if $\rho$ is larger than $\rho_{upper}$, it is not updated anymore. 
The perturbation $\epsilon$ is set to $10^{-5}$, and $\rho_{upper}$ can be set as any constant than $\frac{1}{\epsilon}$, such as $2 \times 10^5$.
We utilize two stopping criterion jointly, including: 
1) the violation of the local consistency constraint, \ie, 
$( \sum_{(i,\alpha)\in E} \frac{\rho_{i \alpha}}{2} \parallel (1+\epsilon) \boldsymbol{\mu}_i - \M_{i\alpha} \boldsymbol{\mu}_{\alpha} \parallel_2^2 )^{\frac{1}{2}}$; 
2) the violation of the equivalence constraint $ (1+\epsilon) \boldsymbol{\mu}_i = \boldsymbol{\upsilon}_{i}$, \ie, 
$(\sum_{i \in V} \frac{\rho_{i}}{2} \parallel (1+\epsilon) \boldsymbol{\mu}_i - \boldsymbol{\upsilon}_{i} \parallel_2^2)^{\frac{1}{2}}$. We set the same threshold $10^{-5}$ for both criterion. If this two violations are lower than $10^{-5}$ simultaneously, then the algorithm stops.

\section{Convergence Analysis}
\label{sec: convergence analysis}

The convergence property of the above ADMM algorithm is demonstrated in Theorem \ref{theorem: convergence}. Due to the space limit, the detailed proof will be presented in {\bf Appendix} \ref{appendix sec: convergence}.

\begin{theorem}
We suppose that $\rho$ is set to be larger than a constant, 
then the variable sequence $\{\y^k, \hat{\x}^k, \boldsymbol{\lambda}^k\}$ generated by the perturbed ADMM algorithm globally converges to $(\y^*, \hat{\x}^*, \boldsymbol{\lambda}^*)$, where $(\y^*, \x^*)$ is the $\epsilon$-KKT point to the LS-LP problem (\ref{eq: LS-LP(x,y)}), as defined in Definition \ref{definition: epsilon kkt solution}. 

Furthermore, according to Definition \ref{definition: KL property}, we assume that 
$\mathcal{L}_{\rho, \epsilon}$ has the KL property at $(\y^*, \x^*, \boldsymbol{\lambda}^*)$ with the concave function $\varphi(s) = c s^{1-p}$, where $p \in [0, 1), c>0$. 
Consequently, we can obtain the following inequalities:
\begin{enumerate}
     \item[(i)] If $p  =0$, then 
     the perturbed ADMM algorithm will converge in finite steps. 
    \item[(ii)] If $p \in (0, \frac{1}{2}]$, then 
     we will obtain the $\epsilon$-KKT solution to the LS-LP problem in at least $O\big(\log_{\frac{1}{\tau}}(\frac{1}{\epsilon})^2\big)$ steps, with $\tau \in (0,1)$ being a small constant, which will be later defined in Appendix \ref{appendix: global convergence}.
    \item[(iii)] If $p \in ( \frac{1}{2}, 1)$, then 
    we will obtain the $\epsilon$-KKT solution to the LS-LP problem in at least $O\big( (\frac{1}{\epsilon})^{\frac{4p-2}{1-p}}\big)$ steps.
\end{enumerate}
\label{theorem: convergence}
\end{theorem}

\begin{proof}
The general structure of the proof consists of the following steps, as follows:
\begin{enumerate}
    \item The perturbed augmented Lagrangian function $\mathcal{L}_{\rho, \epsilon}$ (see (\ref{eq: augmented L of LS-LP(x,y) with perturbation})) is monotonically decreasing along the optimization. 
    \item The variable sequence $\{\y^k, \hat{\x}^k, \boldsymbol{\lambda}^k\}$ is bounded.
    \item The sequence of variable residuals is converged, \ie, 
    $\{ \|$ $\y^{k+1}$ $- \y^k \|, \| \hat{\x}^{k+1} - \hat{\x}^k \|,  \| \boldsymbol{\lambda}^{k+1} - \boldsymbol{\lambda}^k \| \} \rightarrow 0$, as $k \rightarrow \infty$.
    \item The variable sequence $\{ \y^k, \hat{\x}^k, \boldsymbol{\lambda}^k \}$ globally converges to the cluster point $( \y^*, \hat{\x}^*, \boldsymbol{\lambda}^*)$.
    \item $(\y^*, \x^*)$ is the $\epsilon$-KKT point of the LS-LP problem (\ref{eq: LS-LP(x,y)}).
    \item We finally analyze the convergence rate that how many steps are required to achieve the $\epsilon$-KKT point. 
\end{enumerate}
\end{proof}

\begin{table}[tbhp]
\caption{ Benchmark datasets used in the Probabilistic Inference Challenge (PIC 2011) \cite{PIC-2011} and OpenGM 2 \cite{opengm2-ijcv-2015}. C$_1$ to C$_7$ represent: number of  models, average variables, average factors, average edges, average factor sizes  (\ie, the average number of adjacent variables for each factor), average variable states, average factor states.}
\label{table: dataset}
\vspace{-0.15in}
\begin{center}
\scalebox{1}{
\begin{tabular}{ p{.05\textwidth} |  p{.02\textwidth} p{.05\textwidth} p{.05\textwidth} p{.05\textwidth} p{.022\textwidth} p{.017\textwidth} p{.03\textwidth}  }
\hline
 dataset  & C$_1$  & C$_2$   &  C$_3$  & C$_4$  & C$_5$  & C$_6$ &  C$_7$
  \\
\hline
\hline
Seg-2 & 50 & 229.14 & 622.28 &	1244.56 & 2 & 2 & 4
 \\
 \hline
Seg-21 & 50 & 229.14 &	622.28 & 1244.56 &	2 & 21 & 441
 \\
 \hline
Scene & 715 &  182.56 & 488.99 & 977.98 &  2 & 8 & 64
 \\
 \hline
 Grids & 21 & 3142.86 & 6236.19 & 12472.4 & 2 & 2 & 4
 \\
 \hline
Protein & 7 & 14324.7 &	21854.7 & 57680.4 &	2.64 &	2 &	6.56
 \\
 \hline
\end{tabular}
}
\end{center}
\end{table}

\section{Experiments}
\label{sec:experiments}

\subsection{Experimental Settings}
\label{sec: subsec exprimental setting}

\subsubsection{Datasets}
We evaluate on four benchmark datasets from the Probabilistic Inference Challenge (PIC 2011) \cite{PIC-2011} and OpenGM 2 \cite{opengm2-ijcv-2015}, including Segmentation \cite{PIC-2011} , Scene \cite{scene-data-iccv-2009}, Grids \cite{PIC-2011}, and Protein \cite{protein-data-2006}, as shown in Table \ref{table: dataset}.
Segmentation consists of Seg-2 and Seg-21, with different variable states. 
Protein includes higher-order potentials, while others include pairwise potentials.

\subsubsection{Compared Methods}

We compare with different categories of MAP inference methods, including: \textbf{1)} moving making methods, \ie, ICM \cite{ICM-1986};
\textbf{2)} message-passing methods, including belief propagation (BP) \cite{BP-2001} and TRBP \cite{TRBP-2005};
\textbf{3)} polyhedral methods (including LP relaxation based methods), including  dual decomposition using sub-gradient (DD-SG) \cite{DD-SG-2012}, TRWS \cite{TRWS-PAMI-2006}, ADSal \cite{ADSal-2012}, PSDD \cite{PSDD-ICCV-2007} and AD3 \cite{AD3-ICML-2011}\cite{AD3-JMLR-2015}.
\textbf{4)} We also compare with LP-LP, which calls the the active-set method (implemented by {\it linprog} in MATLAB) to optimize $\text{LP}(\boldsymbol{\theta})$. It serves as a baseline to measure the performance of above methods. 
\textbf{5)} The most related work $\ell_2$-Box ADMM (\ie., the special case of $\ell_2$-Box ADMM with $p=2$) algorithm  \cite{WU-lpbox-ADMM-PAMI-2018} is also compared. However, the presented algorithm in \cite{WU-lpbox-ADMM-PAMI-2018} can only handle MRF models with pairwise potentials, which is formulated as a binary quadratic programming (BQP) problem. Thus, $\ell_2$-Box ADMM (hereafter we call it $\ell_2$-Box for clarity) is not compared on Protein, of which models include high-order potentials. 
\textbf{6)} We also compare with two hybrid methods, including method DAOOPT (adopting branch-and-bound method \cite{branch-and-bound-1960} as a sub-routine) \cite{daoopt-2012}\cite{DAOOPT-2012-details} and MPLP-C \cite{mplp-c-uai-2012} (adopting MPLP \cite{MPLP-NIPS-2007} as a sub-routine). The `hybrid' indicates that the method is a combination of an off-the-shelf single method and some heuristic steps. And we call above 5 types as non-hybrid methods. 
Both the proposed LS-LP and $\ell_2$-Box are implemented by MATLAB. 
The following methods are implemented by the author provided C++ package, including: 
PSDD and AD3\footnote{\scriptsize http://www.cs.cmu.edu/~ark/AD3/},  MPLP-C\footnote{\scriptsize https://github.com/opengm/MPLP}, and
DAOOPT\footnote{\scriptsize https://github.com/lotten/daoopt}.
All other methods are implemented through the OpenGM 2 software \cite{opengm2-ijcv-2015}, and we add a prefix ``ogm" before the method name, such as ogm-TRWS. 

In experiments, we set some upper limits: the maximal iteration as 2000 for PSDD and AD3, 500 for $\ell_2$-Box and LS-LP, and 1000 for other methods; for DAOOPT, the memory limit of mini buckets is set as 4000 MB and the upper time limit as 2 hours. 
The parameter tuning of all compared methods (except $\ell_2$-Box) is self-included in their implementations.
Both LS-LP and $\ell_2$-Box are ADMM algorithms, and their hyper-parameters are tuned as follows: the hyper-parameters $\rho_0$, $\eta$ and $\rho_{upper}$ (see implementation details of Section \ref{sec: perturbed ADMM for LS-LP}) are adjusted in the ranges $\{0.05, 0.1, 1, 5,$ $10, 10^2, 10^3, 10^4 \}$, $\{1.01, 1.03, 1.05, 1.1, 1.2\}$ and $\{10^6, 10^8\}$, respectively, and those leading to the higher logPot value are used.

\subsubsection{Evaluation Metrics}

We evaluate the performance of all compared methods using three types of metrics, including the log potential (logPot) values, the solution type, as well as the computational complexity and runtime. 

\vspace{0.3em}
\noindent
{\bf Evaluation using logPot values.} ~
The logPot value indicates the objective value $\langle \boldsymbol{\theta}, \boldsymbol{\mu} \rangle$ of $\text{MAP}(\boldsymbol{\theta})$ (see Eq. (\ref{eq: MAP inference with M_G})). Given that the constraint $\mathcal{M}_G$ of $\text{MAP}(\boldsymbol{\theta})$ is satisfied, the larger logPot value indicates the better inference performance. 
Since LP-LP gives the optimal solution to $\text{LP}(\boldsymbol{\theta})$ (see Eq. (\ref{eq: MAP inference over L(G)})) constrained to $\mathcal{L}_G$, and we know that $\mathcal{M}_G \subseteq \mathcal{L}_G$, then the logPot value of any valid label configurations cannot be larger than that of LP-LP.
Note that in the implementation of OpenGM 2 \cite{opengm2-ijcv-2015}, a rounding method is adopted as the post-processing step to produce the integer solution for the continuous MAP inference methods. 
However, the performance of different MAP inference methods may be significantly changed by rounding.
Thus, for other methods not implemented by OpenGM 2, we report the logPot values of original continuous solutions, without any rounding.

\vspace{0.3em}
\noindent
{\bf Evaluation using solution types.} ~
Since LP-LP, PSDD and AD3 are possible to give continuous solutions, the larger logPot value doesn't always mean the better MAP inference result. 
Thus, we also define four qualitative measures, including {\it valid, uniform, fractional} and {\it approximate}, to intuitively measure the inference quality.
{\it Valid} (V) means that the solution is integer and satisfies the constraints in $\mathcal{L}_G$;
{\it Uniform} (U) denotes that the solution belongs to $\mathcal{L}_G$, but the value is uniform, such as $(0.5, 0.5)$ for the variable node with binary states.
{\it Fractional} (F) indicates that the solution belongs to $\mathcal{L}_G$, but its value is fractional, while not uniform;
{\it Approximate} (A) means that some constraints in $\mathcal{L}_G$ are violated, and its solutions is integer or fractional.
Note that it makes sense to compare the logPot values for different solutions, only when the solutions are valid. In contrast, it makes no sense to compare the logPot values, if the solutions belong to the other three types of measures. 
To better illustrate above measures, we present a brief example on a toy graphical model, as shown in Table \ref{table: measure types}.

\renewcommand{\arraystretch}{1.1}
\begin{table}[t]
\caption{ 
A brief illustration of four types of measures for inference quality, on a toy graphical model with two connected variable nodes.
}
\vspace{-0.15in}
\label{table: measure types}
\begin{center}
\scalebox{0.78}{
\begin{tabular}{ p{.1\textwidth} |  p{0.07\textwidth} p{.07\textwidth} p{.165\textwidth} | p{.08\textwidth} }
\hline
  & Variable 1 & Variable 2 & \hspace{3.5em} Factor & \multirow{2}{*}{Measure}
  \\
  Possible states & \hspace{1em}$\{0,1\}$ & \hspace{1em}$\{0,1\}$ & \hspace{1.5em} $\{00, 01, 10, 11\}$ & 
  \\
\hline
\hline
 & $(1,0)$ & $(0,1)$ & $(0,1,0,0)$ & Valid
\\
 Inferred & $(0.5, 0.5)$ & $(0.5, 0.5)$ & $(0.25, 0.25, 0.25, 0.25)$ & Uniform
\\
 Probability & $(0.2, 0.8)$ & $(0.4, 0.6)$ & $(0.08, 0.12, 0.32, 0.48)$ & Fractional
\\
 & $(0.2, 0.8)$ & $(0.4, 0.6)$ & $(0.16, 0.3, 0.4, 0.14)$ & Approximate
\\
\hline
\end{tabular}}
\end{center}
\vspace{-0.15in}
\end{table}

\vspace{0.3em}
\noindent
{\bf Evaluation using the computational complexity and practical runtime.} ~
The computational complexity and the practical runtime are also important performance measures for MAP inference methods, as shown in Section \ref{sec: runtime}.

\begin{table*}[t]
\caption{ LogPot values of MAP inference solutions on Seg-2, Seg-21 and Scene. Except of PSDD, all other methods give valid solutions. 
The best result among valid solutions in each row is highlighted in bold. Please refer to Section \ref{sec: results on segmentation and scene} for details. 
}
\vspace{-0.15in}
\label{table: results of segmentation}
\begin{center}
\scalebox{0.89}{
\begin{tabular}{ p{.038\textwidth} p{.04\textwidth} | p{.06\textwidth} | p{.058\textwidth}  p{.058\textwidth} | p{.06\textwidth} p{.056\textwidth} p{.06\textwidth} 
p{.06\textwidth} p{.06\textwidth} p{.056\textwidth} p{.056\textwidth} p{.06\textwidth} | p{.06\textwidth} }
\hline
\multicolumn{2}{c|}{\scalebox{0.8}{Method type} $\rightarrow$} &
Baseline & \multicolumn{2}{c|}{\scalebox{1}{Hybrid methods}} & \multicolumn{8}{c|}{Non-hybrid methods} & Proposed
\\
 \multicolumn{2}{c|}{\scalebox{1}{Dataset} $\downarrow$}  & \scalebox{0.9}{LP-LP} & \scalebox{0.9}{DAOOPT} & \scalebox{0.9}{MPLP-C} &  \scalebox{0.9}{ogm-ICM}  & \scalebox{0.9}{ogm-BP} & \scalebox{0.7}{ogm-TRBP} &\scalebox{0.8}{ogm-TRWS} & \scalebox{0.8}{ogm-ADSal} & PSDD & AD3 & $\ell_2$-Box  & LS-LP
  \\
\hline
\hline
\multirow{2}{*}{Seg-2} & mean &  \textbf{-75.5} & \textbf{-75.5} & \textbf{-75.5} & -137.1 & -79 & -76.8 & \textbf{-75.5} & \textbf{-75.5} &  -75.4 & \textbf{-75.5}  & -76.5 & -75.6
\\
& std & 19.63 & 19.63 & 19.63 & 70.1 & 20.24 & 19.36 & 19.24 & 19.24 & 19.77 & 19.63 & 20.3 & 19.69
\\
\hline
 \multirow{2}{*}{\scalebox{0.9}{Seg-21}} & mean &  \textbf{-324.89} &  -325.34 &  \textbf{-324.89} & -393.37 & -330.37 & -328.92 & \textbf{-324.89} & \textbf{-324.89}  &  -325.1 &  \textbf{-324.89} & -344.51 &  \textbf{-324.89}
 \\
 & std & 58.12 & 58.14 & 58.12 & 74.47 & 58.54 & 58.57 &  56.97 & 56.97 & 58.16 & 58.12 & 59.24 & 58.12
 \\
\hline
\multirow{2}{*}{Scene}  & mean & \textbf{866.66} & \textbf{866.66} & \textbf{866.66} & 864.27 & 866.49 & 866.51 &  \textbf{866.66} & \textbf{866.66} & 866.65 & \textbf{866.66}  & 864.11 & \textbf{866.66}
\\
 & std & 109.34 & 109.34 & 109.36 & 109.64 & 109.22 & 109.2 &  109.19 & 109.19 & 109.34 & 109.34 & 108.66 & 109.34
 \\
 \hline
\end{tabular}
}
\end{center}
\end{table*}

\begin{table*}[t]
\caption{ MAP inference results on Grids dataset. 
LP-LP, PSDD and AD3 produce uniform solutions on all models in Grids, while all other methods give valid solutions. Here we only show the logPot of LP-LP as the upper bound of other methods. 
The best logPot among integer solutions in each row is highlighted in bold. The number with in circle 
indicates the performance ranking of each method. Please refer to Section \ref{sec: result on grids} for details. }
\vspace{-1.5em}
\label{table: results of grids}
\begin{center}
\scalebox{0.84}{
\begin{tabular}{ p{.082\textwidth}  | p{.07\textwidth} | p{.079\textwidth} p{.073\textwidth} | p{.072\textwidth} p{.072\textwidth} p{.072\textwidth} p{.072\textwidth} p{.072\textwidth} p{.072\textwidth} p{.072\textwidth} | p{.085\textwidth} }
\hline
\scalebox{0.7}{Method type $\rightarrow$} &  Baseline & \multicolumn{2}{c|}{Hybrid methods} & \multicolumn{7}{c|}{Non-hybrid methods} & Proposed
\\
\cline{1-11}
 Model $\downarrow$  & \scalebox{0.9}{LP-LP} & \scalebox{0.9}{DAOOPT} & \scalebox{0.9}{MPLP-C} &  \scalebox{0.9}{ogm-ICM}  & \scalebox{0.9}{ogm-BP} & \scalebox{0.7}{ogm-TRBP} & \scalebox{0.7}{ogm-DD-SG} &\scalebox{0.7}{ogm-TRWS} & \scalebox{0.7}{ogm-ADSal} & \scalebox{0.9}{$\ell_2$-Box} & \scalebox{0.9}{LS-LP}
  \\
\hline
\hline
M1 & 3736.7 & \textbf{3015.7} \scalebox{0.65}{\circled{1}} & \textbf{3015.7} \scalebox{0.65}{\circled{1}} & 2708.9 \scalebox{0.65}{\circled{5}} & 121.3 \scalebox{0.65}{\circled{9}} & -235.2 \scalebox{0.65}{\circled{10}} & 1286.3 \scalebox{0.65}{\circled{8}} & 2524.9 \scalebox{0.65}{\circled{7}}  & 2605.2 \scalebox{0.65}{\circled{6}} & 2794.8 \scalebox{0.65}{\circled{4}} & 2931.8 \scalebox{0.65}{\circled{3}}
 \\
 \hline
M2 & 3830.3 & \textbf{3051} \scalebox{0.65}{\circled{1}} & 3033.6 \scalebox{0.65}{\circled{2}} & 2567.9 \scalebox{0.65}{\circled{7}} & 276.4 \scalebox{0.65}{\circled{9}} & 19.2 \scalebox{0.65}{\circled{10}} & 1484.7 \scalebox{0.65}{\circled{8}} & 2674.4 \scalebox{0.65}{\circled{5}} & 2670.2 \scalebox{0.65}{\circled{6}} & 2812.4 \scalebox{0.65}{\circled{4}} & 2936.7 \scalebox{0.65}{\circled{3}} 
\\
 \hline
M3 & 5605.1 & \textbf{4517.3} \scalebox{0.65}{\circled{1}} & \textbf{4517.3} \scalebox{0.65}{\circled{1}} & 4067.3 \scalebox{0.65}{\circled{5}} & 332.1 \scalebox{0.65}{\circled{9}} & 14.02 \scalebox{0.65}{\circled{10}} & 1889.7 \scalebox{0.65}{\circled{8}} & 3829.3 \scalebox{0.65}{\circled{7}} & 3884 \scalebox{0.65}{\circled{6}} & 4301.1 \scalebox{0.65}{\circled{4}} &  4408.9 \scalebox{0.65}{\circled{3}}
\\
 \hline
M4 & 5745.5 & \textbf{4563.2} \scalebox{0.65}{\circled{1}} & \textbf{4563.2} \scalebox{0.65}{\circled{1}}  & 3837.12 \scalebox{0.65}{\circled{7}} & 924.5 \scalebox{0.65}{\circled{9}} & -36.7 \scalebox{0.65}{\circled{10}} & 2023.4 \scalebox{0.65}{\circled{8}} & 3894.6 \scalebox{0.65}{\circled{6}} & 4015 \scalebox{0.65}{\circled{5}} & 4202.4 \scalebox{0.65}{\circled{4}} & 4446.6 \scalebox{0.65}{\circled{3}}
\\
 \hline
M5 & 1915.2 & 	\textbf{1542.7} \scalebox{0.65}{\circled{1}} & \textbf{1542.7} \scalebox{0.65}{\circled{1}} & 1318.41 \scalebox{0.65}{\circled{7}} & 481.5 \scalebox{0.65}{\circled{9}} & -47.8 \scalebox{0.65}{\circled{10}} & 807.6 \scalebox{0.65}{\circled{8}} & 1325.5 \scalebox{0.65}{\circled{5}} & 1323.9 \scalebox{0.65}{\circled{6}} & 1427.4 \scalebox{0.65}{\circled{4}} & 1503.2 \scalebox{0.65}{\circled{3}}
\\
 \hline
M6 & 15601.2 & 12662.9 \scalebox{0.65}{\circled{2}} & \textbf{12665.7} \scalebox{0.65}{\circled{1}} & 10753.7 \scalebox{0.65}{\circled{6}} & 2793.5 \scalebox{0.65}{\circled{9}} & 2214.3 \scalebox{0.65}{\circled{10}} & 5051.9 \scalebox{0.65}{\circled{8}} & 10500.8 \scalebox{0.65}{\circled{7}} & 11029 \scalebox{0.65}{\circled{5}} & 11486.2 \scalebox{0.65}{\circled{4}} & 12336.1 \scalebox{0.65}{\circled{3}}
\\
 \hline
M7 & 16291.5 &  13050.7 \scalebox{0.65}{\circled{2}} & \textbf{13054.8} \scalebox{0.65}{\circled{1}} & 10903.8 \scalebox{0.65}{\circled{5}} & 1217.1 \scalebox{0.65}{\circled{9}} & 132.4 \scalebox{0.65}{\circled{10}} & 4634.8 \scalebox{0.65}{\circled{8}} & 10665 \scalebox{0.65}{\circled{7}} & 10870.4 \scalebox{0.65}{\circled{6}} & 11867.6 \scalebox{0.65}{\circled{4}} & 12537.2 \scalebox{0.65}{\circled{3}}
\\
 \hline
M8 & 23401.8 &   \textbf{18952.45} \scalebox{0.6}{\circled{1}} & 18896.8 \scalebox{0.65}{\circled{1}} & 16154.2 \scalebox{0.65}{\circled{6}} & 4314.9 \scalebox{0.65}{\circled{10}} & 5371.1 \scalebox{0.65}{\circled{9}} & 7160 \scalebox{0.65}{\circled{8}} & 16014 \scalebox{0.65}{\circled{7}} & 16276.9 \scalebox{0.65}{\circled{5}} & 17367.5 \scalebox{0.65}{\circled{4}} & 18358.7 \scalebox{0.65}{\circled{3}}
\\
 \hline
M9 & 24437.3 &  \textbf{19538} \scalebox{0.65}{\circled{1}} &	19427.5 \scalebox{0.65}{\circled{2}} & 16334.2 \scalebox{0.65}{\circled{6}} & 3560.8 \scalebox{0.65}{\circled{9}} & -1111 \scalebox{0.65}{\circled{10}} & 7187.3 \scalebox{0.65}{\circled{8}} & 16004.3 \scalebox{0.65}{\circled{7}} & 16508.1 \scalebox{0.65}{\circled{5}} & 17990 \scalebox{0.65}{\circled{4}} & 18785.8  \scalebox{0.65}{\circled{3}}
\\
 \hline
M10  & 3121.2 &  \textbf{2689} \scalebox{0.65}{\circled{1}} & 2688.8 \scalebox{0.65}{\circled{2}} & 2255.38 \scalebox{0.65}{\circled{6}} & 1665.3 \scalebox{0.65}{\circled{8}} & 1582.9 \scalebox{0.65}{\circled{9}} & 1330.7 \scalebox{0.65}{\circled{10}} & 2215.7 \scalebox{0.65}{\circled{7}} & 2369.1 \scalebox{0.65}{\circled{5}} & 2552.6 \scalebox{0.65}{\circled{4}} & 2659.8 \scalebox{0.65}{\circled{3}}
\\
 \hline
M11 & 3231.6 &  \textbf{2714.67} \scalebox{0.65}{\circled{1}} & 2714.52 \scalebox{0.65}{\circled{2}}  & 2258.54 \scalebox{0.65}{\circled{7}}  & 1399.6 \scalebox{0.65}{\circled{8}} & 42.8 \scalebox{0.65}{\circled{10}} & 1285.9 \scalebox{0.65}{\circled{9}} & 2271.5 \scalebox{0.65}{\circled{6}} & 2370.1 \scalebox{0.65}{\circled{5}} & 2556.8 \scalebox{0.65}{\circled{4}} & 2654.9 \scalebox{0.65}{\circled{3}}
\\
 \hline
M12 & 7800.6 &  \textbf{6401.15} \scalebox{0.65}{\circled{1}} & 6396 \scalebox{0.65}{\circled{2}} & 5356.28 \scalebox{0.65}{\circled{6}} & 2033.5 \scalebox{0.65}{\circled{9}} & 1953.1 \scalebox{0.65}{\circled{10}} & 2832.5 \scalebox{0.65}{\circled{8}} & 5282.5 \scalebox{0.65}{\circled{7}} & 5558.8 \scalebox{0.65}{\circled{5}} & 5903 \scalebox{0.65}{\circled{4}} & 6201.2 \scalebox{0.65}{\circled{3}}
\\
 \hline
 M13 & 8078.5 &  \textbf{6472.9} \scalebox{0.65}{\circled{1}} & 6469.7 \scalebox{0.65}{\circled{2}}  & 5425.16 \scalebox{0.65}{\circled{7}} & 1711.5 \scalebox{0.65}{\circled{9}} & 381 \scalebox{0.65}{\circled{10}} & 2814.3 \scalebox{0.65}{\circled{8}} & 5452.8 \scalebox{0.65}{\circled{6}} & 5646.1 \scalebox{0.65}{\circled{5}} & 5923.5 \scalebox{0.65}{\circled{4}} & 6275.4 \scalebox{0.65}{\circled{3}}
\\
 \hline
M14 & 62943 & 	-- & 45813.6 \scalebox{0.65}{\circled{2}} &  43538.9 \scalebox{0.65}{\circled{4}} & 5690.9 \scalebox{0.65}{\circled{9}} & 6426.7 \scalebox{0.65}{\circled{8}} & 18700.4 \scalebox{0.65}{\circled{7}}  & 42274.2 \scalebox{0.65}{\circled{6}} & 43292.5 \scalebox{0.65}{\circled{5}} & 44397.5 \scalebox{0.65}{\circled{3}} & \textbf{48766.1} \scalebox{0.65}{\circled{1}}
\\
 \hline
M15 & 63993.1 &  --	& 47444.4 \scalebox{0.65}{\circled{2}} &  42855 \scalebox{0.65}{\circled{5}} & 4287.1 \scalebox{0.65}{\circled{8}} & 956.4 \scalebox{0.65}{\circled{9}} & 18811.9 \scalebox{0.65}{\circled{7}} & 42535 \scalebox{0.65}{\circled{6}}  & 42918.7 \scalebox{0.65}{\circled{4}} & 44759.5 \scalebox{0.65}{\circled{3}} & \textbf{48657.3} \scalebox{0.65}{\circled{1}}
\\
 \hline
 M16 & 94414.5  & 	-- & 69408.6 \scalebox{0.65}{\circled{2}} &  65081.2 \scalebox{0.65}{\circled{4}}  & 4374.2 \scalebox{0.65}{\circled{9}} & 4656.5 \scalebox{0.65}{\circled{8}} & 27320.6 \scalebox{0.65}{\circled{7}} & 63148.1 \scalebox{0.65}{\circled{6}} & 64401.1 \scalebox{0.65}{\circled{5}} & 66784.2 \scalebox{0.65}{\circled{3}} & \textbf{72993.8} \scalebox{0.65}{\circled{1}}
\\
 \hline
 M17 & 96243.6  & 	-- & 71730.8 \scalebox{0.65}{\circled{2}} & 63768.1 \scalebox{0.65}{\circled{6}} & 13662.7 \scalebox{0.65}{\circled{8}} & -529.3 \scalebox{0.65}{\circled{9}} & 27287.7 \scalebox{0.65}{\circled{7}} & 63885.1 \scalebox{0.65}{\circled{5}} & 64487.9 \scalebox{0.65}{\circled{4}} & 67589.4 \scalebox{0.65}{\circled{3}} & \textbf{73486} \scalebox{0.65}{\circled{1}}
\\
 \hline
 M18 & 12721.3 &  -- & 10445.8 \scalebox{0.65}{\circled{2}}  & 9062.03 \scalebox{0.65}{\circled{5}} & 5198.7 \scalebox{0.65}{\circled{7}} & 4975.4 \scalebox{0.65}{\circled{8}} & 4785.5 \scalebox{0.65}{\circled{9}} & 8793.5 \scalebox{0.65}{\circled{6}}  & 9408.4 \scalebox{0.65}{\circled{4}} & 10015.1 \scalebox{0.65}{\circled{3}} & \textbf{10580.8} \scalebox{0.65}{\circled{1}}
 \\
 \hline
M19 & 12875.6 & 	--	& 10674.1 \scalebox{0.65}{\circled{2}}  & 9214.57 \scalebox{0.65}{\circled{5}} & 5944.6 \scalebox{0.65}{\circled{7}} & 1213.1 \scalebox{0.65}{\circled{9}} & 5328.5 \scalebox{0.65}{\circled{8}} &  8952.4 \scalebox{0.65}{\circled{6}} & 9385.1 \scalebox{0.65}{\circled{4}} &  10163.6 \scalebox{0.65}{\circled{3}} & \textbf{10698.4} \scalebox{0.65}{\circled{1}}
\\
 \hline
M20 & 31809.7 &  --	& 22292.5 \scalebox{0.65}{\circled{3}} & 21527.9 \scalebox{0.65}{\circled{6}} & 5410.9 \scalebox{0.65}{\circled{8}} & 4762 \scalebox{0.65}{\circled{9}} & 9837.3 \scalebox{0.65}{\circled{7}} & 21546.8 \scalebox{0.65}{\circled{5}} & 22109.5 \scalebox{0.65}{\circled{4}} & 22913.3 \scalebox{0.65}{\circled{2}} & \textbf{24834.5} \scalebox{0.65}{\circled{1}}
\\
 \hline
 M21 & 31996.9 &  -- & 24032.4 \scalebox{0.65}{\circled{2}}  &	21529.6 \scalebox{0.65}{\circled{5}} & 4242.3 \scalebox{0.65}{\circled{8}} & 47.6 \scalebox{0.65}{\circled{9}} & 10423.8 \scalebox{0.65}{\circled{7}} & 21195.9 \scalebox{0.65}{\circled{6}} & 21730.3 \scalebox{0.65}{\circled{4}} & 22668.7 \scalebox{0.65}{\circled{3}} & \textbf{24532.8}    \scalebox{0.65}{\circled{1}}
\\
\hline
\end{tabular}
}
\end{center}
\end{table*}

\begin{table*}[thpb]
\caption{ 
LogPot values of MAP inference solutions on Protein dataset. Except for PSDD and AD3, all other methods give integer solutions. 
Both PSDD and AD3 produce mixed types of solutions on all models.
\comment{
The solution types of PSDD on D1 to D8 are: 
$A+6.96\%U+11.75\%F; A+1.47\%U$ $+3.19\%F; A+0.73\%U+2.06\%F; A+13.72\%U+19.67\%F; A+4.35\%U+7.9\%F; A+6.28\%U+10.4\%F; A+0.68\%U+1.89\%F$. 
Those of AD3 are: $A+2.01\%U+21.52F; A+0.55\%U+0.66F; 0.17\%U; A+9.51\%U+32.44\%F; A+0.74\%U+13.7\%F; A+3.41\%U+17.16\%F; A+$  $0.38\%U +0.13\%F$.
}
The best result among valid solutions in each row is highlighted in bold. The number with in circle indicates the performance ranking of each method. 
Please refer to Section \ref{sec: result on protein} for details. }
\vspace{-0.2in}
\label{table: results of protein}
\begin{center}
\scalebox{1.05}{
\begin{tabular}{ p{.09\textwidth}  | 
p{.08\textwidth} | p{.08\textwidth} p{.08\textwidth} p{.08\textwidth} p{.08\textwidth} p{.08\textwidth} p{.08\textwidth} | p{.08\textwidth} }
\hline
\scalebox{0.78}{Method type} $\rightarrow$ &  \scalebox{0.8}{Hybrid methods} & \multicolumn{6}{c|}{Non-hybrid methods} & Proposed
\\
\cline{1-8}
  Model $\downarrow$ &  MPLP-C & \scalebox{0.8}{ogm-ICM} & ogm-BP & \scalebox{0.8}{ogm-TRBP} & \scalebox{0.75}{ogm-DD-SG} & PSDD & AD3  & LS-LP
  \\
\hline
\hline
M1 &  -30181.3 \scalebox{0.7}{\circled{2}} & -32409.9 \scalebox{0.7}{\circled{5}} & -32019.1 \scalebox{0.7}{\circled{4}} & -31671.6 \scalebox{0.7}{\circled{3}} & -33381.2 \scalebox{0.7}{\circled{6}} & -30128.8 & -30143.6  & \textbf{-30165.5} \scalebox{0.7}{\circled{1}}
\\
 \hline
 M2 &  -29305.4 \scalebox{0.7}{\circled{2}} & -32561.3 \scalebox{0.7}{\circled{5}} & -30966.1 \scalebox{0.7}{\circled{3}} & -31253.3 \scalebox{0.7}{\circled{4}}  & -33583.6 \scalebox{0.7}{\circled{6}} & -29307.3 & -29302.6  & \textbf{-29295.4} \scalebox{0.7}{\circled{1}}
\\
 \hline
  M4 & -28952.1 \scalebox{0.7}{\circled{2}} & -32570 \scalebox{0.7}{\circled{5}} & -31031.4 \scalebox{0.7}{\circled{3}} & -31176.6 \scalebox{0.7}{\circled{4}} & -33747.7 \scalebox{0.7}{\circled{6}} & -28952.5 & -28952  & \textbf{-28952} \scalebox{0.7}{\circled{1}}
\\
 \hline
   M5 &  -269567 \scalebox{0.7}{\circled{3}} & \textbf{-256489} \scalebox{0.7}{\circled{1}} & -382766 \scalebox{0.7}{\circled{5}} & -357330 \scalebox{0.7}{\circled{4}} & -553376 \scalebox{0.7}{\circled{6}} & -66132.3   &  -115562  & -267814  \scalebox{0.7}{\circled{2}}
\\
 \hline
   M6 &   -30070.6 \scalebox{0.7}{\circled{2}} & -31699.1 \scalebox{0.7}{\circled{5}} & -30765.2 \scalebox{0.7}{\circled{3}} & -30772.2 \scalebox{0.7}{\circled{4}} & -32952.9 \scalebox{0.7}{\circled{6}} & -30063.6 & -30062.2  & \textbf{-30063.4} \scalebox{0.7}{\circled{1}}
\\
 \hline
   M7 &  -30288.3 \scalebox{0.7}{\circled{2}} & -32562.2 \scalebox{0.7}{\circled{5}}  & -31659.6 \scalebox{0.7}{\circled{3}} & -31791.1 \scalebox{0.7}{\circled{4}}  & -33620.4 \scalebox{0.7}{\circled{6}} & -30248.5 & -30239.8  & \textbf{-30266} \scalebox{0.7}{\circled{1}}
\\
 \hline
   M8 &  -29336.5 \scalebox{0.7}{\circled{2}} & -32617.2 \scalebox{0.7}{\circled{5}} & -31064.7 \scalebox{0.7}{\circled{3}} & -31219.9 \scalebox{0.7}{\circled{4}} & -34549.9  \scalebox{0.7}{\circled{6}} & -29331 & -29336.1  & \textbf{-29334.7} \scalebox{0.7}{\circled{1}}
\\
 \hline
\end{tabular}
}
\end{center}
\end{table*}

\subsection{Results on Segmentation and Scene}
\label{sec: results on segmentation and scene}

The average results on Seg-2, Seg-21 and Scene are shown in Table \ref{table: results of segmentation}. 
LP-LP gives valid solutions on all models, \ie, the best solutions. 
Except for PSDD, all other methods give valid solutions, and their logPot values can not be higher than those of LP-LP. The logPot values of ICM are the lowest, and those of ogm-BP, ogm-TRBP are slightly lower than the best logPot values, while other methods achieve the best logPot values on most models.  
Only PSDD gives approximate solutions (\ie, the constraints in $\mathcal{L}_G$ are not fully satisfied) on some models, specifically, 5 models in Seg-2, 8 models in Seg-21 and 166 models in Scene. 
ogm-DD-SG fails to give solutions on some models of these datasets, thus we ignore it.
Evaluations on these easy models only show that the performance ranking is ogm-ICM < ogm-BP, ogm-TRBP, $\ell_2$-Box < others.

\subsection{Results on Grids}
\label{sec: result on grids}

The results on Grids are shown in Table \ref{table: results of grids}. 
For clarity, we use the model indexes M1 to M21 to indicate the model name to save space in this section. 
The corresponding model names from M1 to M21 are 
grid20x20.f10.uai, grid20x20.f10.wrap. uai,
grid20x20.f15.uai, grid20x20.f15.wrap.uai,
grid20x20. f5.wrap.uai,
grid40x40.f10.uai, grid40x40.f10.wrap.uai,
grid 40x40.f15.uai, grid40x40.f15.wrap.uai, 
grid40x40.f2.uai, grid 40x40.f2.wrap.uai,
grid40x40.f5.uai, grid40x40.f5.wrap.uai, 
grid80  x80.f10.uai, grid80x80.f10.wrap.uai, 
grid80x80.f15.uai, grid  80x80.f15.wrap.uai, 
grid80x80.f2.uai, grid80x80.f2.wrap. uai, 
grid80x80.f5.uai, grid80x80.f5.wrap.uai, respectively. 

The models in Grids are much challenging for LP relaxation based methods, as all models have symmetric pairwise log potentials and very dense cycles in the graph. 
In this case, many vertices of $\mathcal{L}_G$ are uniform solutions $(0.5, 0.5)$. Consequently, the LP relaxation based methods are likely to produce uniform solutions. 
This is verified by that LP-LP, AD3, PSDD give uniform solutions on all models in Grids, \ie, most solutions are $0.5$. 
Thus, we only show the logPot values of LP-LP in Table \ref{table: results of grids}, to provide the theoretical upper-bound of logPot of valid solutions from other methods. 
In contrast, the additional $\ell_2$-sphere constraint in LS-LP excludes the uniform solutions. 
On small scale models M1 to M13, DAOOPT and MPLP-C show the highest logPot values, while LS-LP gives slightly lower values. 
On large scale models M14 to M21, DAOOPT fails to give any result within 2 hours. LS-LP gives the best results, while MPLP-C shows slightly lower results. 
$\ell_2$-Box performs worse than LS-LP, MPLP-C and DAOOPT on most models, while better than all other methods, among which ogm-BP, ogm-TRBP and ogm-DD-SG perform worst. 
These results demonstrate that \textbf{1)} LS-LP is comparable to hybrid methods DAOOPT and MPLP-C, but with much lower computational cost (shown in Section \ref{sec: runtime}); \textbf{2)} LS-LP performs much better than other non-hybrid methods.

\subsection{Results on Protein}
\label{sec: result on protein}

The results on Protein are shown in Table \ref{table: results of protein}.
Different with above three datasets, Protein includes 8 large scale models, and with high-order factors. 
Similarly, we use the model indexes M1 to M8 to indicate the model name to save space in this section. 
The corresponding model names of M1 to M8 are didNotconverge1.uai, didNotconverge2.uai, didNotconverge4.uai, didNotconverge5.uai, didNotconverge6.uai, didNotconverge7.uai, didNotconverge8.uai, respectively.  
As M1 and M3 are the same model, we remove M3 in experiments.
DAOOPT fails to give solutions within 2 hours on all models, and LP-LP cannot produce solutions due to the memory limit. 
ogm-TRWS and $\ell_2$-Box are not evaluated as it cannot handle high-order factors.
LS-LP produces valid integer solutions on all models, and gives the highest logPot values on all models except of M5. 
MPLP-C gives slightly lower logPot values than LS-LP. 
AD3 only produces a fractional solution on M4, while produces approximate and fractional solutions on other models, while PSDD gives approximate and fractional solutions on all models. 
Specifically, the solution types of AD3 on M1 to M8 are: 
$A+2.01\%U+21.52F; A+0.55\%U+0.66F; 0.17\%U; A+9.51\%U+32.44\%F; A+0.74\%U+13.7\%F; A+3.41\%U+17.16\%F; A+$  $0.38\%U +0.13\%F$.
Those of PSDD are: $A+6.96\%U+11.75\%F; A+1.47\%U$ $+3.19\%F; A+0.73\%U+2.06\%F; A+13.72\%U+19.67\%F; A+4.35\%U+7.9\%F; A+6.28\%U+10.4\%F; A+0.68\%U+1.89\%F$. 
Other methods also show much worse performance than LS-LP and MPLP-C. One exception is that ogm-ICM gives the best results on M5, and we find that M5 is the most challenging model for approximated methods.

\renewcommand{\arraystretch}{1.1}
\begin{table}[t]
\caption{ Computational complexities of all compared methods. 
Excluding $\mathcal{E}$, the definitions of all other notations can be found in Section \ref{sec: background}. $\mathcal{E}$ denotes the edge set of the original MRF graph, while $E$ indicates the edge set of the corresponding factor graph. $T$ represents the number of iterations. 
}
\vspace{-0.15in}
\label{table: complexity}
\begin{center}
\scalebox{0.93}{
\begin{tabular}{ p{.09\textwidth} |  p{.37\textwidth}  }
\hline
 Methods  & Complexities
  \\
\hline
\hline
MPLP & $O\big(\sum_i^V |\mathcal{N}_i|^2 \cdot |\mathcal{X}_i| + 2 \sum_{(i,j)}^\mathcal{E} (|\mathcal{X}_i| + |\mathcal{X}_j|) \big)$
 \\
 \hline
MPLP-C & $O\big(T_{\text{outer}} \big[ T_\text{inner} O(\text{MPLP}) + |\mathcal{E}| \big]\big)$
 \\
 \hline
ogm-ICM & $O\big(T [\sum_i^V |\mathcal{X}_i|] \big)$
 \\
 \hline
ogm-BP ~~ ogm-TRBP & $O\big(T \big[\sum_i^V (|\mathcal{N}_i| - 1) \sum_{\alpha}^{\mathcal{N}_i} |\mathcal{X}_{\alpha} | +  \sum_{\alpha}^F (|\mathcal{N}_{\alpha}| - 1) \sum_{i}^{\mathcal{N}_{\alpha}} |\mathcal{X}_{i} | \big] \big)$
\\
 \hline
ogm-TRWS ~~ ogm-ADSal & $O\big( T \big[|\mathcal{E}| \cdot \max_{i \in V} |\mathcal{X}_i| \big] \big)$
\\
 \hline
PSDD ~~~~~ AD3 & $O\big( T \big[ \sum_i^V [ |\mathcal{N}_i| \cdot  |\mathcal{X}_i| ] + \sum_{\alpha}^F |\mathcal{X}_{\alpha}|^3 + \sum_{(i, \alpha)}^E |\mathcal{X}_i| \cdot |\mathcal{X}_{\alpha}| \big] \big)$
\\
 \hline
$\ell_2$-Box & $O\big( T [\sum_i^V |\mathcal{X}_i| ]^3 \big)$
\\
 \hline
LS-LP & $O\big( T \big[ \sum_i^V  |\mathcal{X}_i|  + \sum_{\alpha}^F |\mathcal{X}_{\alpha}|^3 + \sum_{(i, \alpha)}^E |\mathcal{X}_i| \cdot |\mathcal{X}_{\alpha}| \big] \big)$
\\
\hline
\end{tabular}
}
\end{center}
\vspace{-0.15in}
\end{table}

\begin{table*}[tpbh]
\caption{ Iterations and practical runtime on Seg-2, Seg-21 and Scene. 
}
\vspace{-0.1in}
\label{table: runtime of segmentation}
\begin{center}
\scalebox{0.85}{
\begin{tabular}{ p{.038\textwidth} |p{.048\textwidth} p{.045\textwidth} | p{.05\textwidth} p{.056\textwidth} p{.053\textwidth} p{.056\textwidth} p{.056\textwidth} p{.056\textwidth}  
p{.065\textwidth} p{.065\textwidth} p{.05\textwidth} p{.05\textwidth} p{.06\textwidth} p{.05\textwidth} }
\hline
 \multicolumn{3}{c|}{\scalebox{1}{Datasets}}  & \scalebox{0.9}{LP-LP} & \scalebox{0.8}{DAOOPT} & \scalebox{0.8}{MPLP-C} &  \scalebox{0.85}{ogm-ICM}  & \scalebox{0.9}{ogm-BP} & \scalebox{0.7}{ogm-TRBP} &\scalebox{0.8}{ogm-TRWS} & \scalebox{0.8}{ogm-ADSal} & PSDD & AD3  &
 $\ell_2$-Box &  \scalebox{0.9}{LS-LP}
  \\
\hline
\hline
\multirow{4}{*}{Seg-2} & \multirow{2}{*}{iters} & mean &  9.4 & -- & 1.3 & 2.12 & 1000 & 1000 & 18.74 & 16.96 & 632.5 & 70.12  & 78.84 & 84.14
\\
& & std & 1.03 & -- & 0.463 & 0.32 & 0 & 0 & 15.67 & 8.55 & 618.2 & 76.79  & 100.6 & 49.58
\\
\cline{2-15}
&  \multirow{2}{*}{runtime} & mean & 0.096 & 0.54 & 0.033 & 0.007 & 36.51 & 40.73 & 0.191 & 0.7828 & 0.032 & 0.001 & 0.849 & 12.53
\\
&  & std & 0.01 & 0.504 & 0.054 & 0.001 & 0.61 & 0.67 & 0.167 & 0.6373 & 0.027  & 0.006  & 1.044 & 7.39
\\
\hline
\hline
 \multirow{4}{*}{\scalebox{0.9}{Seg-21}} & \multirow{2}{*}{iters} & mean & 16.13 & -- & 1.46 & 76.9  & 1000 & 1000 & 23.4 & 17.58 & 783 & 73.26 & 39.4 & 78.63
 \\
 & & std & 1.55 & -- & 0.646 & 82.2 & 0 & 0 & 20.53 & 11.95 & 672 & 45.09 & 68.1 & 34.65
 \\
\cline{2-15}
 &  \multirow{2}{*}{runtime} & mean & 45.71 & 1311.8 & 0.745 & 0.185 & 1612 & 1891.6  & 16.33 & 57.01 & 1.14 & 0.228  & 8.643 & 342
\\
 &  & std & 9.34 & 1593.2 & 1.1 & 0.062 & 25.05 & 25.7 & 12.89 & 44.51 & 1.1 & 0.114 & 14.71 &  147
 \\
\hline
\hline
\multirow{4}{*}{Scene} & \multirow{2}{*}{iters} & mean & 12.56 & -- & 1.32 & 319.5 & 1000 & 1000 & 15.03 & 10.94 & 791 & 71.59 & 43.03 & 75.36
\\
 &  & std & 1.04 & -- & 0.48 & 58.89 & 0 & 0 & 12.21 & 11.15 & 766 & 82.83 & 37.47 & 52.47
 \\
\cline{2-15}
 &  \multirow{2}{*}{runtime} & mean & 1.907 & 82.89 & 0.072 & 0.081 & 229.67 & 265.45 & 1.7 & 5.03 & 0.066 & 0.029  & 0.777 & 29.23
\\
 & & std & 0.266 & 18.46 & 0.111 & 0.011 & 15.46 & 17.6 & 1.3 & 5.88 & 0.071 & 0.024 & 0.681 & 20.36
 \\
 \hline
\end{tabular}
}
\end{center}
\end{table*}

\begin{table*}[tpbh]
\caption{ Iterations and practical runtime on Grids.}
\label{table: runtime of grids}
\vspace{-0.15in}
\begin{center}
\scalebox{0.9}{
\begin{tabular}{ p{.025\textwidth} p{.05\textwidth} | p{.055\textwidth} p{.055\textwidth} p{.055\textwidth} p{.055\textwidth} p{.055\textwidth} p{.055\textwidth}  
p{.055\textwidth} p{.055\textwidth}  p{.06\textwidth} p{.05\textwidth} }
\hline
 \multicolumn{2}{c|}{\scalebox{1}{Models}}  & \scalebox{0.85}{DAOOPT} & \scalebox{0.9}{MPLP-C} &  \scalebox{0.9}{ogm-ICM}  & \scalebox{0.9}{ogm-BP} & \scalebox{0.7}{ogm-TRBP} & \scalebox{0.7}{ogm-DD-SG} &\scalebox{0.8}{ogm-TRWS} & \scalebox{0.8}{ogm-ADSal} & $\ell_2$-Box &  LS-LP
  \\
\hline
\hline
 \multirow{2}{*}{M1} & iters & -- & 420 & 195 & 1000 & 1000 & 1000 & 14 & 39 & 22 & 234
\\
   & runtime & 8  & 303.9  & 0.012 & 40.1 & 40.89 & 86.9 & 0.13 & 2.79 & 0.317 & 41.4 
\\
\hline
 \multirow{2}{*}{M2} & iters & -- & 1000 & 192 &  1000 & 1000  & 1000 & 18 & 37 & 35 & 227
\\
   & runtime & 10  & 613.3  & 0.013 & 42.71 & 42.13 & 91.1 & 0.18 & 2.72 & 0.478 & 39.4 
\\
\hline
 \multirow{2}{*}{M3} & iters & -- & 686 & 195 &  1000 & 1000  & 1000 & 13 & 36 & 36 & 155
\\
   & runtime & 8  &  331.5 & 0.022 & 40.2 & 39.28 & 86.4 & 0.12 & 2.52 & 0.412 & 27   
\\
\hline
 \multirow{2}{*}{M4} & iters & -- & 1000 & 193 &  1000 & 1000  & 1000 & 13 & 42 & 27 & 434 
\\
   & runtime & 10  &  523.5 & 0.013 & 42.72 & 45.53 & 91.5 & 0.17 & 3.16 & 0.318 & 75.3   
\\
\hline
 \multirow{2}{*}{M5} & iters & -- & 1000 & 193 &  1000 & 1000  & 1000 & 21 & 37 & 26 & 265
\\
   & runtime & 11  & 468.7  & 0.013 & 42.77 & 48.07 & 91.2 & 0.02 & 2.73 & 0.291 & 45.7 
\\
\hline
 \multirow{2}{*}{M6} & iters & -- & 1000 & 741 &  1000 & 1000  & 1000 & 19 & 37 & 18 & 271
\\
   & runtime & 7200  & 840.3  & 0.052 & 166.6 & 182.3 & 362.3 & 0.78 & 11.1 & 2.19 & 46.5 
\\
\hline
 \multirow{2}{*}{M7} & iters & -- & 1000 & 750 &  1000 & 1000  & 1000 & 18 & 37 & 16 & 204   
\\
   & runtime & 7200  & 1040.5  & 0.053 & 171.6 & 187.2 & 371.7 & 0.7 & 11.4 & 1.96 & 35.3 
\\
\hline
 \multirow{2}{*}{M8} & iters & -- & 1000 & 730 & 1000 & 1000  & 1000 & 17 & 37 & 29 & 407 
\\
   & runtime &   7200  & 917.1  & 0.072 & 166.6 & 182.1 & 362.4 & 0.68 & 10.9 & 3.47 & 66.1  
\\
\hline
 \multirow{2}{*}{M9} & iters & -- & 1000 & 739 &  1000 & 1000  & 1000 & 16 & 36 & 31 & 327  
\\
   & runtime &  7200  &  1096.8 & 0.053 & 171.7 & 189.5 & 372.3 & 0.85 & 10.6 & 3.80 & 56.3  
\\
\hline
 \multirow{2}{*}{M10} & iters & -- & 521 & 738 &  1000 & 1000  & 1000 & 45 & 38 & 45 & 500 
\\
   & runtime &  7200  & 273.9  & 0.072 & 166.6 & 182.1 & 362.4 & 2.14 & 10.6 & 5.68 & 88.3 
\\
\hline
 \multirow{2}{*}{M11} & iters & -- & 1000 & 765 & 1000 & 1000  & 1000 & 36 & 44 & 39 & 390  
\\
   & runtime &  7200  & 843.1  & 0.053 & 171.7 & 187.7 & 372.2 & 1.57 & 12.6 & 5.13 & 68.4   
\\
\hline
 \multirow{2}{*}{M12} & iters & -- & 1000 & 727 &  1000 & 1000  & 1000 & 26 & 38 & 19 & 314  
\\
   & runtime &  7200  & 824.8  & 0.073 & 166.7 & 187.6 & 361.5 & 1.05 & 11 & 2.25 & 52.5 
\\
\hline
 \multirow{2}{*}{M13} & iters & -- & 1000 & 755 &  1000 & 1000  & 1000 & 34 & 41 & 17 & 493 
\\
   & runtime &  7200  & 999.7  & 0.052 & 171.8 & 188.4 & 372.3 & 1.44 & 11.9 & 2.14 & 85.9 
\\
\hline
 \multirow{2}{*}{M14} & iters & -- & 1000 & 3029 &  1000 & 1000  & 1000 & 21 & 37 & 12 & 290 
\\
   & runtime &  --  & 1956.4  & 0.218 & 676.5 & 756 & 1474 & 3.52 & 43.5 & 21.72 & 50.5 
\\
\hline
 \multirow{2}{*}{M15} & iters & -- & 1000 & 2991 &  1000 & 1000  & 1000 & 19 & 36 & 12 & 202  
\\
   & runtime &  -- & 2028.9  & 0.218 & 687.2 & 770.2 & 1495 & 3.37 & 42 & 21.65 & 33.7 
\\
\hline
 \multirow{2}{*}{M16} & iters & -- & 1000 & 3038 &  1000 & 1000  & 1000 & 15 & 35 & 20 & 435  
\\
   & runtime &  --  & 2110.4 & 0.217 & 676.8 & 762.2 & 1471 & 2.58 & 41.5 & 36.53 & 76    
\\
\hline
 \multirow{2}{*}{M17} & iters & -- & 1000 & 3048 &  1000 & 1000  & 1000 & 20 & 37 & 15 & 354 
\\
   & runtime &  --  & 2238.8  & 0.221 & 686.8 & 767.6 & 1492.2 & 3.37 & 43 & 26.93 &  64
\\
\hline
 \multirow{2}{*}{M18} & iters & -- & 1000 & 3173 &  1000 & 1000  & 1000 & 54 & 53 & 35 & 250  
\\
   & runtime & -- & 1956.3  & 0.221 & 677 & 756.1 & 1475.2 & 8.87 & 63.7 & 60.28 & 43.5
\\
\hline
 \multirow{2}{*}{M19} & iters & -- & 1000 & 3095 &  1000 & 1000  & 1000 & 109 & 45 & 37 & 250 
\\
   & runtime &  -- & 1953.5  & 0.221 & 687.4 & 773 & 1493.8 & 19.13 & 50.7 & 63.93 & 39.9 
\\
\hline
 \multirow{2}{*}{M20} & iters & -- & 1000 & 3038 &  1000 & 1000  & 1000 & 30 & 38 & 13 & 316
\\
   & runtime &  -- & 1974.2  & 0.206 & 676.4 & 756 & 1473.2 & 4.99 & 42.5 & 23.06 &  55.4
\\
\hline
 \multirow{2}{*}{M21} & iters & -- & 1000 & 2985 &  1000 & 1000  & 1000 & 22 & 37 & 12 & 280 
\\
   & runtime &  -- & 2167.3  & 0.218 & 687.2 & 767.5 & 1494.4 & 3.75 & 40.6 & 21.54 & 49.6
\\
\hline
\end{tabular}
}
\end{center}
\vspace{-0.9em}
\end{table*}

\begin{table*}[tbhp]
\caption{ Iterations and practical runtime on Protein.}
\label{table: runtime of protein}
\vspace{-1em}
\begin{center}
\scalebox{1.0}{
\begin{tabular}{ p{.025\textwidth}  p{.06\textwidth} | 
p{.055\textwidth} p{.055\textwidth} p{.055\textwidth} p{.055\textwidth}
p{.055\textwidth} p{.045\textwidth} p{.045\textwidth} p{.05\textwidth} }
\hline
   \multicolumn{2}{c|}{\scalebox{1}{Models}} &  \scalebox{0.9}{MPLP-C} &  \scalebox{0.9}{ogm-ICM}  & \scalebox{0.9}{ogm-BP} & \scalebox{0.8}{ogm-TRBP} & \scalebox{0.8}{ogm-DD-SG} &\scalebox{1}{PSDD} & \scalebox{1}{AD3} &  \scalebox{1}{LS-LP} 
  \\
\hline
\hline
 \multirow{2}{*}{M1} & iters  & 499 & 797 & 1000 & 1000 & 1000 &  2000  &  2000 &  161  
\\
   & runtime  & 2320 & 0.39 & 32019 & 31671 & 33381 & 14.14  &  20.28 &  1746   
\\
\hline
 \multirow{2}{*}{M2} & iters & 478 & 634 & 1000 & 1000 & 1000  &   2000  &  2000 &  98   
\\
   & runtime & 2399 & 0.44 & 30966 & 31253 & 33583 &    7.74  &  3.87  &  1076  
\\
\hline
 \multirow{2}{*}{M4} & iters  & 24 & 859 &  1000 & 1000 & 1000  &   2000  &  1524 &  140  
\\
   & runtime  & 80 & 0.51 & 31031 & 31177 & 33748 &   5.73  &  3.32 &  1535   
\\
\hline
 \multirow{2}{*}{M5} & iters  & 500 & 5275 &  1000 & 1000 & 1000  &   2000  &  2000 & 1000   
\\
   & runtime  & 2248 & 0.8 & 382766 & 357330 & 553376 &   22.62  &  28.36 &  10787   
\\
\hline
 \multirow{2}{*}{M6} & iters  & 433 & 597 & 1000 & 1000 & 1000  &   2000  &  2000 & 482 
\\
   & runtime  & 2427 & 0.36 & 30765 & 30772 & 32953 &   12.96  &  14.12 &  5314   
\\
\hline
 \multirow{2}{*}{M7} & iters  & 475 & 836 &  1000 & 1000 & 1000  &   2000  &  2000 & 145  
\\
   & runtime  & 2399 & 0.43 & 31660 & 31791 & 33620 &   13.55  &  18.63 &  1989.7  
\\
\hline
 \multirow{2}{*}{M8} & iters  & 16 & 778 &  1000 & 1000 & 1000  &   2000  &  2000 & 186  
\\
   & runtime & 80 & 0.51 & 31065 & 31220 & 34550 &   5.13  &  3.31  & 1566  
\\
\hline
\end{tabular}
}
\end{center}
\vspace{-1em}
\end{table*}

\subsection{Comparisons on Computational Complexities and Practical Runtime}
\label{sec: runtime}

\textbf{Computational complexities} of all compared methods (except of LP-LP and DAOOPT) are summarized in Table \ref{table: complexity}. 
As there is no clear conclusion of the complexity of the active-set algorithm for linear programming, the complexity of LP-LP is not presented.
DAOOPT \cite{DAOOPT-2012-details} is combination of 6 sequential sub-algorithms and heuristic steps, thus its computational complexity cannot be computed. 
As these complexities depend on the graph structure (\ie, $V, E, \mathcal{E}, F, \mathcal{N}_i,$ $\mathcal{N}_{\alpha}, \mathcal{X}_i, \mathcal{X}_{\alpha}$),
it is impossible to give a fixed ranking of them. 
However, it is notable that the complexity of LS-LP is linear w.r.t. the number of variables, factors and edges (\ie, $|V|, |F|, |E|$). Moreover, as the sub-problems w.r.t. variables/ factors/edges in LS-LP are independent, they can be solved in parallel. 
However, the complexity of LS-LP is super-linear w.r.t. the state space $|\X_{\alpha}|$ and $|\X_i|$. 
Thus, the proposed LS-LP method is suitable for graphical models with very large-scale variables and dense connections, but with modest state spaces for variables and factors.

\vspace{0.3em}
\noindent
\textbf{Practical runtime}. 
Due to the dependency of the computational complexity on the graph structure, the practical runtime of these methods will vary significantly on different graphs. 
In the following, we present the practical runtime of all compared methods on above evaluated datasets. 
To test the runtime fairly, we run all methods at the same machine, and only run one experiment at the same time. 
The iterations and practical runtime on different datasets are 
shown respectively in Table \ref{table: runtime of segmentation} for Segmentation and Scene, Table \ref{table: runtime of grids} for Grids and Table \ref{table: runtime of protein} for Protein.

As shown in Table \ref{table: runtime of segmentation}, on the small and easy models, the runtime of LP-LP, MPLP-C, ogm-ICM, PSDD and AD3 are very small, and the runtime of ogm-BP, ogm-TRBP, ogm-TRWS, ogm-ADSal, LS-LP and $\ell_2$-Box are larger, while the runtime of DAOOPT are the largest.  
Besides, the iterations of AD3 and LS-LP are much smaller than the one of PSDD, given the fact that their similar computational complexities per iteration. 
The iterations of LP-LP and MPLP-C are also provided, but they are incomparable with PSDD, AD3 and LS-LP. The complexity of each iteration in LP-LP depends on the problem size $|\boldsymbol{\mu}|$ (see Eq. (\ref{eq: MAP inference over L(G)})). 
In MPLP-C, each outer iteration includes 100 iterations of MPLP and adding violated constraints. The complexity of each MPLP iteration is stable, but the complexity of adding violated constraints varies significantly in different outer iterations. 

In terms of the comparison on Grids (see Table \ref{table: runtime of grids}), the iterations and runtime of LP-LP are smaller than those of other methods, but it gives uniform solutions on all models. As demonstrated Section \ref{sec: results on segmentation and scene}, LP-LP, PSDD and AD3 produce uniform solutions on all models, thus we also don't present their iterations and runtime. 
The runtime of DAOOPT on small models (\ie, M1 to M5) are small, and it achieves 7200 seconds (the upper limit) on M6 to M13, while it cannot give any solution in 7200 seconds on M14 to M21.  
For MPLP-C, both iterations and runtime are large on all models. Note that the runtime per iteration of MPLP-C becomes larger along with the model scale. 
For ICM, the iteration is large on multiple models, but with very small runtime, as its complexity per iteration is low. 
Both message passing methods, including ogm-BP and ogm-TRBP, achieve the upper limit of iterations. It demonstrate that their convergence is very slow. 
The convergence of ogm-DD-SG is also slow, and its runtime per iteration is even higher than above two message passing methods. 
Two LP relaxation based methods, including ogm-TRWS and ogm-ADSal, converge in a few iterations, and are of small runtime. 
In contrast, the iterations of LS-LP are only larger than those of ogm-TRWS, ogm-ADSal and $\ell_2$-Box, while smaller than other methods. 
And, the runtime per iteration of LS-LP is similar with that of ogm-BP, while smaller than those of other methods except of ogm-ICM and ogm-TRWS. Considering that C++ is much more efficient than MATLAB, if LS-LP is also implemented by C++, its runtime should be much lower than those of most compared methods. In other words, the computational complexity of LS-LP is much smaller than most compared methods.
Note that the runtime per iteration of $\ell_2$-Box increases along with the model size. For example, its runtime per iteration on M1 is $0.014$ seconds, while that on M21 is $1.795$ seconds. In contrast, the runtime per iteration of LS-LP are rather stable. As shown in Table \ref{table: complexity}, the complexity per iteration of $\ell_2$-Box is $O\big( [\sum_i^V |\mathcal{X}_i|]^3 \big)$. Obviously, $\ell_2$-Box is difficult to apply to large-scale models.  
In contrast, LS-LP is conducted based on the decomposition of the factor graph to independent factors and variables, due to which the parallel computation is allowed. Thus, the scalability of LS-LP is much better than $\ell_2$-Box for MAP inference.

In terms of the comparison on Protein (see Table \ref{table: runtime of protein}), 
the ascending ranking of runtime is ogm-ICM, PSDD, AD3, MPLP-C, LS-LP, ogm-BP, ogm-TRBP, ogm-DD-SG. 
Although the runtime of PSDD and AD3 are very small, but they only give {\it approximate} solutions in 2000 iterations on all models, except for AD3 on M4. 
For MPLP-C, the iterations and runtime are small on M4 and M8, but very large on all other 5 models. 
In contrast, although the runtime of LS-LP is much larger than those of PSDD and AD3, its iterations are much smaller on all models. If LS-LP is also implemented by C++, its practical runtime will be much smaller than those of PSDD and AD3.

In summary, the above comparisons on iterations and practical runtime demonstrate:
\begin{enumerate}
\item LS-LP converges much faster than most compared methods, except of ogm-TRWS, ogm-ADSal and $\ell_2$-Box. On large-scale models (see Table \ref{table: results of protein}), the complexities per iteration of LS-LP, PSDD and AD3 are similar, and are much smaller than those of other methods (except of ICM).
\item The complexity of MPLP-C is larger than PSDD, AD3 and LS-LP, while smaller than DAOOPT. But its practical iterations and runtime vary significantly on different models. 
Both the complexity and practical runtime of DAOOPT are much larger than other methods. 
\item Considering the performance of the MAP inference results presented in Section \ref{sec: results on segmentation and scene}, \ref{sec: result on grids} and \ref{sec: result on protein}, we conclude that LS-LP shows very competitive performance compared to state-of-the-art MAP inference methods.
\end{enumerate}

\vspace{-1em}
\subsection{Discussions} 
We obtain three conclusions from above experiments evaluated on different types of models. \textbf{1)} Compared with the hybrid methods including DAOOPT and MPLP-C, the performance of LS-LP is comparable. However, the computational cost of LS-LP is much lower, as the hybrid methods adopt LP relaxation based methods as sub-routines.  
\textbf{2)} Compared with LP relaxation based methods, especially PSDD and AD3, LS-LP always give valid solutions, without rounding; and, the logPot values of LS-LP are much higher, with similar computational cost. 
\textbf{3)} Compared to $\ell_2$-Box, which can be only applied to models with pairwise potentials, LS-LP is applied to any type of models. Besides, the decomposition of the factor graph allows for the parallel computations with respect to each factor and each variable, while $\ell_2$-Box solves a QP problem over the whole MRF model. Thus, LS-LP is a much better choice than $\ell_2$-Box for MAP inference.
\textbf{4)} Compared to other approximated methods, LS-LP always shows much better performance in difficult models (\eg, Grids and Protein).

\section{Conclusions}
\label{sec: conclusion}

In this work, we proposed an novel formulation of MAP inference, called $\ell_2$-sphere linear program (LS-LP). Starting from the standard linear programming (LP) relaxation, we added the $\ell_2$-sphere constraint onto variable nodes. The intersection between the $\ell_2$-sphere constraint and the local marginal polytope $\mathcal{L}_G$ in LP relaxation is proved to be the exact set of all valid integer label configurations. Thus, the proposed LS-LP problem is equivalent to the original MAP inference problem. 
By adding a sufficiently small perturbation $\epsilon$ onto the objective function and constraints, we proposed a perturbed ADMM algorithm for optimizing the LS-LP problem. 
Although the $\ell_2$-sphere constraint is non-convex, we proved that the ADMM algorithm will globally converge to the $\epsilon$-KKT point of the LS-LP problem. 
The analysis of convergence rate is also presented. 
Experiments on three benchmark datasets show the competitive performance of LS-LP compared to state-of-the-art MAP inference methods.

\comment{
\vspace{1em}
\noindent 
{\bf Acknowledgement} 
Baoyuan Wu was partially supported by Tencent AI Lab and King Abdullah University of Science
and Technology (KAUST). 
Li Shen was supported by Tencent AI Lab. 
Bernard Ghanem was supported by the King Abdullah University of Science
and Technology (KAUST) Office of Sponsored Research
through the Visual Computing Center (VCC) funding.
Tong Zhang was supported by the Hong Kong University of Science and Technology (HKUST). 
}

\newpage
\onecolumn
\appendix

\section{Convergence Analysis}
\label{appendix sec: convergence}

To facilitate the convergence analysis, 
here we rewrite some equations and notations firstly defined in Sections \ref{sec: perturbed ADMM for LS-LP}. 
Problem (\ref{eq: MAP inference over L(G) and v over S}) can be simplified to the following general shape, as follows
\begin{flalign}
& \text{LS-LP}(\boldsymbol{\theta})
  =  
  \min_{\x, \y} f(\x) + h(\y), ~\text{s.t.} ~  \A \x = \B \y.
  \label{appendix eq: LS-LP(x,y)}
\end{flalign}
Our illustration for (\ref{eq: LS-LP(x,y)}) is separated into three parts, as follows:
\begin{enumerate}
    \item {\bf Variables}. $\x = [ \boldsymbol{\mu}_1; \ldots; \boldsymbol{\mu}_{|V|} ] \in \mathbb{R}^{\sum_{i}^{ V} |\mathcal{X}_i|}$, and it concatenates all variable nodes $\boldsymbol{\mu}_V$.  
    $\y = [\y_1; \ldots; \y_{|V|}]$ with $\y_i = [\boldsymbol{\upsilon}_i; \boldsymbol{\mu}_{\alpha_{i,1}}; \ldots; \boldsymbol{\mu}_{\alpha_{i,|\mathcal{N}_i|}}] \in \mathbb{R}^{|\mathcal{X}_i| + \sum_{\alpha}^{\mathcal{N}_i} |\mathcal{X}_{\alpha}|}$.
    $\y$ concatenates all factor nodes $\boldsymbol{\mu}_V$ and the extra variable nodes $\boldsymbol{\upsilon}$; $\y_i$ concatenates the factor nodes and the extra variable node  connected to the $i$-th variable node $\boldsymbol{\mu}_i$. 
    $\mathcal{N}_i$ indicates the set of neighborhood factor nodes connected to the $i$-th variable node; 
    the subscript $\alpha_{i,j}$ indicates the $j$-th factor connected to the $i$-th variable, with $i \in V$ and $j \in \mathcal{N}_i$. 
    \item {\bf Objective functions}. 
    $f(\x)$ $ = \w_{\x}^\top \x$ with $\w_{\x} = - [\boldsymbol{\theta}_1; \ldots; $ $\boldsymbol{\theta}_{|V|}]$. $h(\y) = g(\y) + \w_{\y}^\top \y$, with $\w_{\y} = [\w_1; \ldots;$ $\w_{|V|}]$ with $\w_{i} = -[\boldsymbol{0}; \frac{1}{|\mathcal{N}_{\alpha_{i,1}}|} \boldsymbol{\theta}_{\alpha_{i,1}};$ 
    $\ldots; \frac{1}{|\mathcal{N}_{\alpha_{i,|\mathcal{N}_i|}}|} \boldsymbol{\theta}_{\alpha_{i,|\mathcal{N}_i|}}]$, and $\mathcal{N}_{\alpha} = \{ i \mid (i, \alpha) \in E\}$ being the set of neighborhood variable nodes connected to the $\alpha$-th factor. 
    $g(\y) = \mathbb{I}(\boldsymbol{\upsilon} \in \mathcal{S}) + \sum_{\alpha \in F} \mathbb{I}(\boldsymbol{\mu}_{\alpha} \in \Delta^{|\mathcal{X}_{\alpha}|})$, with $\mathbb{I}(a)$ being the indicator function: $\mathbb{I}(a)=0$ if $a$ is true, otherwise $\mathbb{I}(a)=\infty$.
    \item {\bf Constraint matrices}. 
    The constraint matrix $\A = \text{diag}($ $\A_1, \ldots, \A_i, \ldots, \A_{|V|})$ with $\A_i = [\mathbf{I}_{|\mathcal{X}_i|}; \ldots; \mathbf{I}_{|\mathcal{X}_i|} ] \in \{0,1\}^{(|\mathcal{N}_i| +1)|\mathcal{X}_i| \times  |\mathcal{X}_i|}$. 
    $\B = \text{diag}(\B_1, \ldots,$ $ \B_i, \ldots, \B_{|V|})$, with $\B_i = \text{diag}(\mathbf{I}_{|\mathcal{X}_i|}, \M_{i, \alpha_{i,1}}, \ldots, \M_{i, \alpha_{i, |\mathcal{N}_i|}} )$. 
    $\A$ summarizes all constraints on $\boldsymbol{\mu}_V$, while $\B$ collects all constraints on $\boldsymbol{\mu}_F$ and $\boldsymbol{\upsilon}$. 
\end{enumerate}
Note that Problem (\ref{eq: LS-LP(x,y)}) has a clear structure with two groups of variables, corresponding the augmented factor graph (see Fig. \ref{fig: factor graph}(c)).

According to the analysis presented in \cite{wotao-yin-arxiv-2015}, a sufficient condition to ensure the global convergence of the ADMM algorithm for the problem  
$\text{LS-LP}(\boldsymbol{\theta})$ is that 
$\text{Im}(\B) \subseteq \text{Im}(\A)$, with $\text{Im}(\A)$ 
being the image of $\A$,
  \ie, the column space of $\A$.  However, $\A$ in (\ref{appendix eq: LS-LP(x,y)}) is full column rank, rather than full row rank, while $\B$ is full row rank. 
To satisfy this sufficient condition, we introduce a sufficiently small perturbation to both the objective function and the constraint in (\ref{appendix eq: LS-LP(x,y)}), as follows
\begin{flalign}
& \text{LS-LP}(\boldsymbol{\theta}; \epsilon)
  =  
  \min_{\hat{\x}, \y} \hat{f}(\hat{\x}) + h(\y), ~\text{s.t.} ~  \hat{\A} \hat{\x} = \B \y,
  \label{appendix eq: LS-LP(x,y) with perturbation}
\end{flalign}
where $\hat{\A} = [\A, \epsilon \mathbf{I}]$ with a sufficiently small constant $\epsilon > 0$, then $\hat{\A}$ is full row rank. 
$\hat{\x} = [\x; \bar{\x}]$, with $\bar{\x} = [\bar{\x}_1; \ldots; \bar{\x}_{|V|}] \in \mathbb{R}^{\sum_i^{V} (|\mathcal{N}_i|+1) |\mathcal{X}_i|}$ and  $\bar{\x}_i = [\boldsymbol{\mu}_i; \ldots; \boldsymbol{\mu}_i] \in \mathbb{R}^{(|\mathcal{N}_i|+1) |\mathcal{X}_i|}$. 
$\hat{f}(\hat{\x}) = f(\x) + \frac{1}{2}\epsilon \hat{\x}^\top \hat{\x}$. 
Consequently, $\text{Im}(\hat{\A}) \equiv \text{Im}(\B) \subseteq \mathbb{R}^{\text{rank of } ~ \hat{\A}}$, as both $\hat{\A}$ and $\B$ are full row rank. Then, the sufficient condition $\text{Im}(\B) \subseteq \text{Im}(\hat{\A})$ holds.

The augmented Lagrangian function of (\ref{appendix eq: LS-LP(x,y) with perturbation}) is formulated as
\begin{flalign}
\mathcal{L}_{\rho, \epsilon}(\hat{\x}, \y, \boldsymbol{\lambda}) = \hat{f}(\hat{\x}) + h(\y) + \boldsymbol{\lambda}^\top (\hat{\A} \hat{\x} - \B \y) + \frac{\rho}{2} \| \hat{\A} \hat{\x} - \B \y \|_2^2
\end{flalign}
The updates of the ADMM algorithm to optimize (\ref{appendix eq: LS-LP(x,y) with perturbation}) are as follows
\begin{flalign}
\begin{cases}
 \y^{k+1} & = \argmin_{\y} \mathcal{L}_{\rho, \epsilon}(\hat{\x}^k, \y, \boldsymbol{\lambda}^k),
 \\
 \hat{\x}^{k+1} & = \argmin_{\hat{\x}} \mathcal{L}_{\rho, \epsilon}(\hat{\x}, \y^{k+1}, \boldsymbol{\lambda}^k),
 \\
 \boldsymbol{\lambda}^{k+1} & = \boldsymbol{\lambda}^k + \rho (\hat{\A} \hat{\x}^{k+1} - \B \y^{k+1}).
 \end{cases}
\end{flalign}
The optimality conditions of the variable sequence $(\y^{k+1}, \hat{\x}^{k+1}, \boldsymbol{\lambda}^{k+1})$ generated above are 
\begin{flalign}
 & \B^\top \boldsymbol{\lambda}^k + \rho \B^\top (\hat{\A} \hat{\x}^k - \B \y^{k+1}) =  \B^\top \boldsymbol{\lambda}^{k+1} - \rho  \B^\top \hat{\A} (\hat{\x}^{k+1} - \hat{\x}^k) \in \partial h(\y^{k+1}), 
 \label{appendix eq: optimality 1}
 \\
 & \nabla \hat{f}(\hat{\x}^{k+1}) + \hat{\A}^\top \boldsymbol{\lambda}^k + \rho \hat{\A}^\top (\hat{\A} \hat{\x}^{k+1} - \B \y^{k+1}) = \nabla \hat{f}(\hat{\x}^{k+1}) + \hat{\A}^\top \boldsymbol{\lambda}^{k+1} = \boldsymbol{0},
 \label{appendix eq: optimality 2}
 \\
 & \frac{1}{\rho} (\boldsymbol{\lambda}^{k+1} - \boldsymbol{\lambda}^{k}) = \hat{\A} \hat{\x}^{k+1} - \B \y^{k+1}. 
 \label{appendix eq: optimality 3}
\end{flalign}

The convergence of this perturbed ADMM algorithm for the LS-LP problem is summarized in Theorem \ref{theorem: convergence}. 
The detailed proof is presented in the following sub-sections sequentially. 
Note that hereafter $\| \cdot \|$ indicates the $\ell_2$ norm for a vector, or the Frobenius norm for a matrix; $\mathcal{A}_1 \succeq \mathcal{A}_2$ represents that $\mathcal{A}_1 - \mathcal{A}_2$ is positive semi-definite, with $\mathcal{A}_1, \mathcal{A}_2$ being square matrices;
$\nabla$ denotes the gradient operator, $\nabla^2$ means the Hessian operator, and $\partial$ is the sub-gradient operator; 
$\I$ represents the identity matrix with compatible shape.

\subsection{Properties}
\label{sec: subsection property}
In this section, we present some important properties of the objective function and constraints in (\ref{appendix eq: LS-LP(x,y) with perturbation}), which will be used in the followed convergence analysis.

\vspace{3pt}
\noindent
{\bf Properties on objective functions (P1)}
\begin{itemize}
    \item (P1.1) $f, h$ and $\mathcal{L}_{\rho, \epsilon}$ are semi-algebraic, lower semi-continuous functions and satisfy Kurdyka-Lojasiewicz (KL) property, and $h$ is closed and proper
    \item (P1.2)  There exist $\mathcal{Q}_1, \mathcal{Q}_2$ such that $\mathcal{Q}_1 \succeq \nabla^2 \hat{f}(\hat{\x}) \succeq \mathcal{Q}_2$, $\forall \hat{\x}$
    \item (P1.3) $\lim \inf_{\| \hat{\x} \| \rightarrow \infty} \| \nabla \hat{f}(\hat{\x}) \| = \infty$
\end{itemize}

\vspace{3pt}
\noindent
{\bf Properties on constraint matrices (P2)} 
\begin{itemize}
    \item (P2.1) There exists $\sigma > 0$ such that $\hat{\A} \hat{\A}^\top \succeq \sigma \I$
    \item (P2.2) $\mathcal{Q}_2 + \rho \hat{\A}^\top \hat{\A} \succeq \delta \I$ for some $\rho, \delta > 0$, and $\rho > \frac{1}{\epsilon} $
    \item (P2.3) There exists $\mathcal{Q}_3 \succeq [\nabla^2 \hat{f}(\hat{\x})]^2, \forall \hat{\x}$, and $\delta \I \succ \frac{2}{\sigma \rho} \mathcal{Q}_3$ 
    \item (P2.4) Both $\hat{\A}$ and $\B$ are full row rank, and  $\text{Im}(\hat{\A}) \equiv \text{Im}(\B) \subseteq \mathbb{R}^{\text{rank of }~ \hat{\A}}$ 
\end{itemize}

\noindent
{\bf Remark}. 
{\bf (1)} 
Although the definition of KL property (see Definition \ref{definition: KL property}) is somewhat complex, but it holds for many widely used functions, according to \cite{kl-examples-2013}. Typical functions satisfying KL property includes: {\it a)} real analytic functions, and any polynomial function such as $\| \mathbf{H} \x - \mathbf{b} \|$ belongs to this type;  
{\it b)} locally strongly convex functions, such as 
the logistic loss function $\log(1+\exp(-\x))$;
{\it c)} semi-algebraic functions, such as $\| \x \|_1, \| \x \|_2,\| \x \|_{\infty}, \| \mathbf{H} \x - \mathbf{b} \|_1, \| \mathbf{H} \x - \mathbf{b} \|_2, \| \mathbf{H} \x - \mathbf{b} \|_{\infty}$ and the indicator function $\mathbb{I}(\cdot)$. 
It is easy to verify that P1.1 holds in our problem.
{\bf (2)} 
Here we provide an instantiation of above hyper-parameters satisfying above properties. 
Firstly, it is easy to obtain that $\nabla^2 \hat{f}(\hat{\x}) = \epsilon \I$, 
and
$\hat{\A} \hat{\A}^\top = [\A, \epsilon \I] [\A, \epsilon \I]^\top = \A \A^\top + \epsilon^2 \I  \succ \epsilon^2 \I$, 
as well as
$\rho \hat{\A}^\top \hat{\A} \succeq \epsilon \I$, when $\epsilon$ is small enough and $\rho > \frac{1}{\epsilon}$ (\eg, $\rho = \frac{2}{\epsilon}$). 
Then, the values $\mathcal{Q}_1 = \mathcal{Q}_2 = \epsilon \I, \mathcal{Q}_3 = \epsilon^2 \I, \delta = 2 \epsilon, \sigma = \epsilon^2$ satisfy P1.2, P2.1, P2.2 and P2.3. 
Without loss of generality, we will adopt these specific values for these hyper-parameters to simplify the following analysis, while only keeping $\rho$ and $\epsilon$.

\subsection{Decreasing of $\mathcal{L}_{\rho, \epsilon}(\y^k, \hat{\x}^k, \boldsymbol{\lambda}^{k})$}
In this section, we firstly prove the decreasing property of the augmented Lagrangian function, \ie, 
\begin{flalign}
 \mathcal{L}_{\rho, \epsilon}(\y^k, \hat{\x}^k, \boldsymbol{\lambda}^k) > \mathcal{L}_{\rho, \epsilon}(\y^{k+1}, \hat{\x}^{k+1}, \boldsymbol{\lambda}^{k+1}), \forall k.
\end{flalign}

Firstly, utilizing P2.1, P2.3 and (\ref{appendix eq: optimality 2}), we obtain that
\begin{flalign}
 \epsilon^2 \| \boldsymbol{\lambda}^{k+1} - \boldsymbol{\lambda}^{k} \|_2^2 
 \leq 
 \| \hat{\A}(\boldsymbol{\lambda}^{k+1} - \boldsymbol{\lambda}^{k}) \|_2^2
 =
 \| \nabla \hat{f}(\hat{\x}^{k+1}) - \nabla \hat{f}(\hat{\x}^k) \|_2^2 
 = 
 \epsilon^2 \| \hat{\x}^{k+1} - \hat{\x}^k \|_2^2.
 \label{appendix eq: dual_residual^2 <= x_residual^2}
\end{flalign}

Then, we have 
\begin{flalign}
 & \mathcal{L}_{\rho, \epsilon}(\y^{k+1}, \hat{\x}^{k+1}, \boldsymbol{\lambda}^{k+1}) - 
 \mathcal{L}_{\rho, \epsilon}(\y^{k+1}, \hat{\x}^{k+1}, \boldsymbol{\lambda}^{k})
 = 
 (\boldsymbol{\lambda}^{k+1} - \boldsymbol{\lambda}^{k})^\top (\hat{\A} \hat{\x}^{k+1} - \B \y^{k+1})
 \label{appendix eq: L(yk+1,xk+1,dualk+1)-L(yk+1,xk+1,dualk)}
 \\
 = & 
 \frac{1}{\rho} \| \boldsymbol{\lambda}^{k+1} - \boldsymbol{\lambda}^{k} \|_2^2
 \leq \frac{1}{\rho} \| \hat{\x}^{k+1} - \hat{\x}^k \|_2^2
 \nonumber
\end{flalign} 

According to P1.2 and P2.2, $\mathcal{L}_{\rho, \epsilon}(\y^{k+1}, \hat{\x}, \boldsymbol{\lambda}^{k})$ is strongly convex with respect to $\hat{\x}$, with the parameter of at least $2 \epsilon$. Then, we have 
\begin{flalign}
 & \mathcal{L}_{\rho, \epsilon}(\y^{k+1}, \hat{\x}^{k+1}, \boldsymbol{\lambda}^{k}) - 
 \mathcal{L}_{\rho, \epsilon}(\y^{k+1}, \hat{\x}^{k}, \boldsymbol{\lambda}^{k})
 \leq 
  - \epsilon \| \hat{\x}^{k+1} - \hat{\x}^k \|_2^2.
  \label{appendix eq: L(yk+1,xk+1,dualk)-L(yk+1,xk,dualk)}
\end{flalign} 

As $\y^{k+1}$ is the minimal solution of $\mathcal{L}_{\rho, \epsilon}(\y, \hat{\x}^{k}, \boldsymbol{\lambda}^{k})$, it is easy to know
\begin{flalign}
 & \mathcal{L}_{\rho, \epsilon}(\y^{k+1}, \hat{\x}^k, \boldsymbol{\lambda}^{k}) - 
 \mathcal{L}_{\rho, \epsilon}(\y^k, \hat{\x}^{k}, \boldsymbol{\lambda}^{k})
 \leq  0.
 \label{appendix eq: L(yk+1,xk,dualk)-L(yk,xk,dualk)}
\end{flalign} 

Combining (\ref{appendix eq: L(yk+1,xk+1,dualk+1)-L(yk+1,xk+1,dualk)}), (\ref{appendix eq: L(yk+1,xk+1,dualk)-L(yk+1,xk,dualk)}) and (\ref{appendix eq: L(yk+1,xk,dualk)-L(yk,xk,dualk)}), we have 
\begin{flalign}
 \mathcal{L}_{\rho, \epsilon}(\y^{k+1}, \hat{\x}^{k+1}, \boldsymbol{\lambda}^{k+1}) - 
 \mathcal{L}_{\rho, \epsilon}(\y^{k}, \hat{\x}^{k}, \boldsymbol{\lambda}^{k}) 
 \leq 
 (\frac{1}{\rho} - \epsilon ) \| \hat{\x}^{k+1} - \hat{\x}^k \|_2^2 
 < 0, 
 \label{appendix eq: L(yk+1,xk+1,dualk+1)-L(yk,xk,dualk) < 0}
\end{flalign}
where the last inequality utilizes P2.3 and $\rho > \frac{1}{\epsilon}$.

\subsection{Boundedness of $\{\y^k, \hat{\x}^k, \boldsymbol{\lambda}^k\}$}

Next, we prove the boundedness of $\{\y^k, \hat{\x}^k, \boldsymbol{\lambda}^k\}$. 
We suppose that $\rho$ is large enough such that there is $0<\gamma<\rho$ with 
\begin{flalign}
 \inf_{\hat{\x}} \big( \hat{f}(\hat{\x}) - \frac{1}{2 \epsilon^2 \gamma} \| \nabla \hat{f}(\hat{\x}) \|_2^2 \big) = f^* > - \infty. 
\end{flalign}
According to (\ref{appendix eq: L(yk+1,xk+1,dualk+1)-L(yk,xk,dualk) < 0}), for any $k \geq 1$, we have 
\begin{flalign}
&\mathcal{L}_{\rho, \epsilon}(\y^k, \hat{\x}^k, \boldsymbol{\lambda}^k) 
=
\hat{f}(\hat{\x}^k) + h(\y^k) + \frac{\rho}{2} \| \hat{\A} \hat{\x}^k - \B \y^k + \frac{\boldsymbol{\lambda}^k}{\rho} \|_2^2 - \frac{1}{2\rho} \| \boldsymbol{\lambda}^k \|_2^2 
\leq 
 \mathcal{L}_{\rho, \epsilon}(\y^1, \hat{\x}^1, \boldsymbol{\lambda}^1) < \infty .
\label{appendix eq: L(yk,xk,dualk) <= L(y1,x1,dual1)}
\end{flalign}
Besides, according to P2.1, we have 
\begin{flalign}
 \epsilon^2 \| \boldsymbol{\lambda}^k \|_2^2 \leq \| \hat{\A}^\top \boldsymbol{\lambda}^k \|_2^2 = \| \nabla \hat{f}(\hat{\x}^k) \|_2^2. 
 \label{appendix eq: sigma * dual^2 <= gradient f(x)^2}
\end{flalign}
Plugging (\ref{appendix eq: sigma * dual^2 <= gradient f(x)^2}) into (\ref{appendix eq: L(yk,xk,dualk) <= L(y1,x1,dual1)}), we obtain that
\begin{flalign}
\infty & > \hat{f}(\hat{\x}^k) + h(\y^k) + \frac{\rho}{2} \| \hat{\A} \hat{\x}^k - \B \y^k + \frac{\boldsymbol{\lambda}^k}{\rho} \|_2^2 - \frac{1}{2 \epsilon^2 \rho} \| \nabla \hat{f}(\hat{\x}^k)  \|_2^2 
 \geq f^* + \frac{\frac{1}{\gamma} - \frac{1}{\rho}}{2 \epsilon^2} \| \nabla \hat{f}(\hat{\x}^k) \|_2^2 + h(\y^k) + \frac{\rho}{2} \| \hat{\A} \hat{\x}^k - \B \y^k + \frac{\boldsymbol{\lambda}^k}{\rho} \|_2^2 .
\end{flalign}
According to the coerciveness of $\nabla \hat{f}(\hat{\x}^k)$ (\ie, P1.3), we obtain that $\hat{\x}^k < \infty, \forall k$, \ie, the boundedness of $\{\hat{\x}^k\}$. From (\ref{appendix eq: sigma * dual^2 <= gradient f(x)^2}), we know the boundedness of $\{\boldsymbol{\lambda}^k\}$. 
Besides, according to P2.4, $\{\hat{\A}\hat{\x}^k\}$ is also bounded. 
From (\ref{appendix eq: optimality 3}), we obtain the boundedness of $\{\B\y^k\}$. Considering the full row rank of $\B$ (\ie, P2.4), the boundedness of $\{\y^k\}$ is proved.

\subsection{Convergence of Residual}
\label{sec: subsec Convergence of residual}

According to the boundedness of $\{\y^k, \hat{\x}^k, \boldsymbol{\lambda}^k\}$, there is a sub-sequence $\{\y^{k_i}, \hat{\x}^{k_i}, \boldsymbol{\lambda}^{k_i}\}$ that converges to a cluster point $\{\y^*, \hat{\x}^*, \boldsymbol{\lambda}^*\}$. Considering the lower semi-continuity of $\mathcal{L}_{\rho, \epsilon}$ (\ie, P1.1), we have 
\begin{flalign}
& \underset{i \rightarrow \infty}{\lim \inf} ~ \mathcal{L}_{\rho, \epsilon}(\y^{k_i}, \hat{\x}^{k_i}, \boldsymbol{\lambda}^{k_i}) \geq \mathcal{L}_{\rho, \epsilon}(\y^*, \hat{\x}^*, \boldsymbol{\lambda}^*) > - \infty.
\label{appendix eq: lower semicontinuous of L_rho}
\end{flalign}

Summing (\ref{appendix eq: L(yk+1,xk+1,dualk+1)-L(yk,xk,dualk) < 0}) from $k = M, \ldots, N-1$ with $M \geq 1$, we have 
\begin{flalign}
& \mathcal{L}_{\rho, \epsilon}(\y^N, \hat{\x}^N, \boldsymbol{\lambda}^N) - \mathcal{L}_{\rho, \epsilon}(\y^M, \hat{\x}^M, \boldsymbol{\lambda}^M) 
\label{appendix eq: L^N - L^M}
\leq (\frac{1}{\rho} - \epsilon) \sum_{k=M}^{N-1} \| \hat{\x}^{k+1} - \hat{\x}^k \|_2^2 < 0. 
\end{flalign}
Then, by setting $N = k_i$ and $M=1$, we have 
\begin{flalign}
&  \mathcal{L}_{\rho, \epsilon}(\y^{k_i}, \hat{\x}^{k_i}, \boldsymbol{\lambda}^{k_i}) - \mathcal{L}_{\rho, \epsilon}(\y^1, \hat{\x}^1, \boldsymbol{\lambda}^1) 
\leq 
(\frac{1}{\rho} - \epsilon) \sum_{k=1}^{k_i-1} \| \hat{\x}^{k+1} - \hat{\x}^k \|_2^2. 
\label{eq: sum ||x^k+1 - x^k||_2^2 > Lk - L1}
\end{flalign}
Taking limit on both sides of the above inequality, we obtain
\begin{flalign}
& -\infty < \mathcal{L}_{\rho, \epsilon}(\y^*, \hat{\x}^*, \boldsymbol{\lambda}^*) - \mathcal{L}_{\rho, \epsilon}(\y^1, \hat{\x}^1, \boldsymbol{\lambda}^1) 
\leq 
(\frac{1}{\rho} - \epsilon) \sum_{k=1}^{\infty} \| \hat{\x}^{k+1} - \hat{\x}^k \|_2^2 < 0. 
\end{flalign}
It implies that 
\begin{flalign}
\lim_{k \rightarrow \infty }\| \hat{\x}^{k+1} - \hat{\x}^k \| = 0.
\label{appendix eq: x-residual converges to 0}
\end{flalign}
Besides, according to (\ref{appendix eq: dual_residual^2 <= x_residual^2}), it is easy to obtain that 
\begin{flalign}
\lim_{k \rightarrow \infty }\| \boldsymbol{\lambda}^{k+1} - \boldsymbol{\lambda}^k \| = 0.
\label{appendix eq: dual-residual converges to 0}
\end{flalign}
Moreover, utilizing $\B \y^{k+1} = \hat{\A} \hat{\x}^{k+1} - \frac{1}{\rho}(\boldsymbol{\lambda}^{k+1} - \boldsymbol{\lambda}^k)$ from (\ref{appendix eq: optimality 3}), we have 
\begin{flalign}
\| \B(\y^{k+1} - \y^{k}) \| \leq \| \hat{\A} (\hat{\x}^{k+1} - \hat{\x}^k) \| + \frac{1}{\rho} \| (\boldsymbol{\lambda}^{k+1} - \boldsymbol{\lambda}^k) \| + \frac{1}{\rho} \| (\boldsymbol{\lambda}^{k} - \boldsymbol{\lambda}^{k-1}) \|.
\label{appendix eq: || B(yk+1 - yk || <=}
\end{flalign}
Besides, as shown in Lemma 1 in \cite{wotao-yin-arxiv-2015}, the full row rank of $\B$ (\ie, P1.4) implies that
\begin{flalign}
\| \y^{k+1} - \y^k \| \leq \Bar{M} \| \B (\y^{k+1} - \y^k) \|,
\label{appendix eq: yk - y* < B(yk - y*)}
\end{flalign}
where $\Bar{M} > 0$ is a constant. 
Taking limit on both sides of (\ref{appendix eq: || B(yk+1 - yk || <=}) and utilizing (\ref{appendix eq: yk - y* < B(yk - y*)}), we obtain 
\begin{flalign}
\lim_{k \rightarrow \infty }\| \y^{k+1} - \y^k \| = 0.
\label{appendix eq: y-residual converges to 0}
\end{flalign}
Combining (\ref{appendix eq: x-residual converges to 0}), (\ref{appendix eq: dual-residual converges to 0}) and (\ref{appendix eq: y-residual converges to 0}), we obtain that 
\begin{flalign}
\lim_{k \rightarrow \infty } \| \y^{k+1} - \y^k \|_2^2 +\| \hat{\x}^{k+1} - \hat{\x}^k \|_2^2 + \| \boldsymbol{\lambda}^{k+1} - \boldsymbol{\lambda}^k \|_2^2 = 0.
\label{appendix eq: x_residual + y-residual + dual-residual converges to 0}
\end{flalign}
By setting $k+1 = k_i$, plugging (\ref{appendix eq: x-residual converges to 0}) into (\ref{appendix eq: optimality 1}) and (\ref{appendix eq: dual-residual converges to 0}) into (\ref{appendix eq: optimality 2}), and taking limit $k_i \rightarrow \infty$, we obtain the KKT conditions. It tells that the cluster point $(\y^*, \hat{\x}^*)$ is the KKT point of $\text{$\text{LS-LP}(\boldsymbol{\theta; \epsilon})$}$ (\ie, (\ref{appendix eq: LS-LP(x,y) with perturbation})).

\subsection{Global Convergence}
\label{appendix: global convergence}
Inspired by the analysis presented in \cite{admm-nonconvex-siam-2015}, 
in this section we will prove the following conclusions:
\begin{itemize}
    \item $\sum_{k=1}^{\infty} \| \hat{\x}^{k+1} - \hat{\x}^k \| < \infty$;
\item $\{\y^k, \hat{\x}^k, \boldsymbol{\lambda}^k\}$ converges to $(\y^*, \hat{\x}^*,  \boldsymbol{\lambda}^*)$;
\item $(\y^*, \hat{\x}^*)$ is the KKT point of (\ref{appendix eq: LS-LP(x,y) with perturbation}).
\end{itemize}

Firstly, utilizing the optimality conditions (\ref{appendix eq: optimality 1}, \ref{appendix eq: optimality 2}, \ref{appendix eq: optimality 3}),  we have that 
\begin{flalign}
 \partial_{\y} \mathcal{L}_{\rho, \epsilon}(\hat{\x}^{k+1}, \y^{k+1}, \boldsymbol{\lambda}^{k+1})
& = ~\partial h(\y^{k+1}) - \B^\top \boldsymbol{\lambda}^{k+1} - \rho \B^\top  (\hat{\A} \hat{\x}^{k+1} - \B \y^{k+1}) 
\ni  - \B^\top (\boldsymbol{\lambda}^{k+1} - \mathbf{\lambda}^k) - \rho \B^\top \hat{\A} (\hat{\x}^{k+1} - \hat{\x}^k),
\label{appendix eq: partial L_rho over y^k+1}
\\
\nabla_{\hat{\x}} \mathcal{L}_{\rho, \epsilon}(\hat{\x}^{k+1}, \y^{k+1}, \boldsymbol{\lambda}^{k+1})
& = \nabla_{\hat{\x}} \hat{f}(\hat{\x}^{k+1}) + \hat{\A}^\top \boldsymbol{\lambda}^{k+1} + \rho \hat{\A}^\top (\hat{\A} \hat{\x}^{k+1} - \B \y^{k+1}) = \hat{\A}^\top (\boldsymbol{\lambda}^{k+1} - \mathbf{\lambda}^k), 
\label{appendix eq: gradient L_rho over x^k+1}
\\
\nabla_{\boldsymbol{\lambda}} \mathcal{L}_{\rho, \epsilon}(\hat{\x}^{k+1}, \y^{k+1}, \boldsymbol{\lambda}^{k+1}) 
& = \hat{\A} \hat{\x}^{k+1} - \B \y^{k+1} 
= \frac{1}{\rho} (\boldsymbol{\lambda}^{k+1} - \mathbf{\lambda}^k).
\label{appendix eq: gradient L_rho over lambda^k+1}
\end{flalign}
Further, combining with (\ref{appendix eq: dual_residual^2 <= x_residual^2}), there exists a constant $C>0$ such that
\begin{flalign}
 & \text{dist}\big(0, \partial_{(\y, \hat{\x}, \boldsymbol{\lambda})} \mathcal{L}_{\rho, \epsilon}(\y^{k+1},\hat{\x}^{k+1},  \boldsymbol{\lambda}^{k+1})\big) \leq C \| \hat{\x}^{k+1} - \hat{\x}^k \|,
 \label{appendix eq: distance between 0 and partial L_rho over y^k+1 is lower than C x_residual}
\end{flalign}
where $\text{dist}(\cdot, \cdot)$ denotes the distance between a vector and a set of vectors. Hereafter we denote $\partial_{(\y, \hat{\x}, \boldsymbol{\lambda})} \mathcal{L}_{\rho, \epsilon}(\y^{k+1},\hat{\x}^{k+1},  \boldsymbol{\lambda}^{k+1})$ as $\partial \mathcal{L}_{\rho, \epsilon}(\y^{k+1},\hat{\x}^{k+1},  \boldsymbol{\lambda}^{k+1})$ for clarity. 
Besides, the relation (\ref{appendix eq: L(yk+1,xk+1,dualk+1)-L(yk,xk,dualk) < 0}) implies that there is a constant $D \in (0, \epsilon - \frac{1}{\rho})$ such that 
\begin{flalign}
\mathcal{L}_{\rho, \epsilon}(\y^{k}, \hat{\x}^{k}, \boldsymbol{\lambda}^{k}) - \mathcal{L}_{\rho, \epsilon}(\y^{k+1}, \hat{\x}^{k+1}, \boldsymbol{\lambda}^{k+1}) \geq 
  D \|  \hat{\x}^{k+1} - \hat{\x}^k \|_2^2.
  \label{appendix eq: L(yk, xk, dualk) - L(yk+1, xk+1, dualk+1) <= D(xk+1 - xk)}
\end{flalign}
Moreover, the relation (\ref{appendix eq: lower semicontinuous of L_rho}) implies that $\{ \mathcal{L}_{\rho, \epsilon}(\y^{k}, \hat{\x}^{k}, \boldsymbol{\lambda}^{k}) \}$ is lower bounded along the convergent sub-sequence $\{ (\y^{k_i}, \hat{\x}^{k_i}, \boldsymbol{\lambda}^{k_i}) \}$. 
Combining with the its decreasing property, the limit of $\{ \mathcal{L}_{\rho, \epsilon}(\y^{k}, \hat{\x}^{k}, \boldsymbol{\lambda}^{k}) \}$ exists. 
Thus, we will show that
\begin{flalign}
\underset{k \rightarrow \infty}{\lim}  \mathcal{L}_{\rho, \epsilon}(\y^{k}, \hat{\x}^{k}, \boldsymbol{\lambda}^{k}) = l^* := \mathcal{L}_{\rho, \epsilon}(\y^{*}, \hat{\x}^*, \boldsymbol{\lambda}^{*}). 
\label{appendix eq: limit of L_rho}
\end{flalign}
To prove it, we utilize the fact that $\y^{k+1}$ is the minimizer of $ \mathcal{L}_{\rho, \epsilon}(\y, \hat{\x}^{k}, \boldsymbol{\lambda}^{k}) $, such that
\begin{flalign}
\mathcal{L}_{\rho, \epsilon}(\y^{k+1}, \hat{\x}^{k}, \boldsymbol{\lambda}^{k}) 
\leq 
\mathcal{L}_{\rho, \epsilon}(\y^{*}, \hat{\x}^{k}, \boldsymbol{\lambda}^{k}).
\end{flalign}
Combining the above relation, (\ref{appendix eq: x_residual + y-residual + dual-residual converges to 0}) and the continuity of $\mathcal{L}_{\rho, \epsilon}$ w.r.t. $\hat{\x}$ and $\boldsymbol{\lambda}$, the following relation holds along the sub-sequence $\{ (\y^{k_i}, \hat{\x}^{k_i}, \boldsymbol{\lambda}^{k_i}) \}$ that converges to $ (\y^{*}, \hat{\x}^{*}, \boldsymbol{\lambda}^{*}) $, 
\begin{flalign}
\underset{i \rightarrow \infty}{\lim \sup} ~ \mathcal{L}_{\rho, \epsilon}(\y^{k_i+1}, \hat{\x}^{k_i+1}, \boldsymbol{\lambda}^{k_i+1}) \leq 
\mathcal{L}_{\rho, \epsilon}(\y^{*}, \hat{\x}^{*}, \boldsymbol{\lambda}^{*}). 
\label{appendix eq: sup L(x^ki+1, y^ki+1, z^ki+1) <= L(x^*, y^*, z^*)}
\end{flalign}
According to (\ref{appendix eq: x_residual + y-residual + dual-residual converges to 0}), the sub-sequence 
$\{ (\y^{k_i+1}, \hat{\x}^{k_i+1}, \boldsymbol{\lambda}^{k_i+1}) \}$ also converges to 
$ (\y^{*}, \hat{\x}^{*}, \boldsymbol{\lambda}^{*}) $. Then, utilizing the lower semi-continuity of $\mathcal{L}_{\rho, \epsilon}$, we have 
\begin{flalign}
\underset{i \rightarrow \infty}{\lim \inf} ~ \mathcal{L}_{\rho, \epsilon}(\y^{k_i+1}, \hat{\x}^{k_i+1}, \boldsymbol{\lambda}^{k_i+1}) \geq 
\mathcal{L}_{\rho, \epsilon}(\y^{*}, \hat{\x}^{*}, \boldsymbol{\lambda}^{*}).
\label{appendix eq: inf L(x^ki+1, y^ki+1, z^ki+1) <= L(x^*, y^*, z^*)}
\end{flalign}
Combining (\ref{appendix eq: sup L(x^ki+1, y^ki+1, z^ki+1) <= L(x^*, y^*, z^*)}) with (\ref{appendix eq: inf L(x^ki+1, y^ki+1, z^ki+1) <= L(x^*, y^*, z^*)}), we know the existence of the limit of the sequence $\{ \mathcal{L}_{\rho, \epsilon}(\y^k, \hat{\x}^k, \boldsymbol{\lambda}^k) \}$, which proves the relation (\ref{appendix eq: limit of L_rho}).

As $\mathcal{L}_{\rho, \epsilon}$ is KL function, according to Definition \ref{definition: KL property}, it has the following properties:
\begin{itemize}
    \item There exist a constant $\eta \in (0, \infty]$, a continuous concave function $\varphi : [0, \eta) \rightarrow \mathbb{R}_{+}$, as well as a neighbourhood $\mathcal{V}$ of $(\y^*, \hat{\x}^*, \boldsymbol{\lambda}^*)$.
 $\varphi$ is differentiable on $(0, \eta)$ with positive derivatives. 
    \item For all $(\y, \hat{\x}, \boldsymbol{\lambda}) \in \mathcal{V}$ satisfying $l^* < \mathcal{L}_{\rho, \epsilon}(\y, \hat{\x}, \boldsymbol{\lambda}) < l^* + \eta$, we have
\begin{flalign}
\varphi'( \mathcal{L}_{\rho, \epsilon}(\y, \hat{\x}, \boldsymbol{\lambda}) - l^* ) \text{dist}(0, \partial \mathcal{L}_{\rho, \epsilon}(\y, \hat{\x}, \boldsymbol{\lambda}) ) \geq 1.
\label{appendix eq: KL property}
\end{flalign}
\end{itemize}
Then, we define the following neighborhood sets:
\begin{flalign}
 & \hspace{-0em} \mathcal{V}_{\zeta} := \bigg\{  (\y, \hat{\x}, \boldsymbol{\lambda}) ~ \bigg\vert ~ \| \hat{\x} - \hat{\x}^* \| < \zeta, 
 \|\y - \y^*\| < \Bar{M}(\|\hat{\A}\|+1) \zeta , 
 \| \boldsymbol{\lambda} - \boldsymbol{\lambda}^* \| < \zeta \bigg\}  \subseteq \mathcal{V}
\label{appendix eq: neighborhood set B_mu}
\\
& \mathcal{V}_{\zeta, \hat{\x}} := \big\{  \hat{\x} ~ \big\vert ~ \| \hat{\x} - \hat{\x}^* \| < \zeta \big\},
\label{appendix eq: neighborhood set B_mu,x}
\end{flalign}
where $\zeta > 0$ is a small constant.

Utilizing the relations (\ref{appendix eq: optimality 2}) and (\ref{appendix eq: optimality 3}), as well as P2.1, we obtain that for any $k\geq 1$, the following relation holds:
\begin{flalign}
&\epsilon^2 \| \boldsymbol{\lambda}^k - \boldsymbol{\lambda}^* \|_2^2
\leq 
\| \hat{\A}^\top (\boldsymbol{\lambda}^k - \boldsymbol{\lambda}^*) \|_2^2 = \| \triangledown \hat{f}(\hat{\x}^k) - \triangledown \hat{f}(\hat{\x}^*) \|_2^2 
=
\epsilon^2 \| \hat{\x}^k - \hat{\x}^* \|_2^2. 
\label{appendix eq: (zk - z*)^2 is lower than lambda_max(Q3) (xk - x*)^2}
\end{flalign}
Also, the relations (\ref{appendix eq: optimality 2}) and (\ref{appendix eq: optimality 3}) imply that for any $k\geq 1$, we have 
\begin{flalign}
 \| \B (\y^k - \y^*) \| = \| \hat{\A} (\hat{\x}^k - \hat{\x}^*) - \frac{1}{\rho}(\boldsymbol{\lambda}^k - \boldsymbol{\lambda}^{k-1}) \| 
 \leq 
 \| \hat{\A} \| \| (\hat{\x}^k - \hat{\x}^*) \| + \frac{1}{\rho} \| \boldsymbol{\lambda}^k - \boldsymbol{\lambda}^{k-1} \|.
 \label{appendix eq: B(yk - y*) < A(xk -x*) + 1/rho (dual^k - dual^k-1)}
\end{flalign}
Moreover, the relation (\ref{appendix eq: x_residual + y-residual + dual-residual converges to 0}) implies that 
$\exists N_0 \geq 1$ such that $\forall k \geq N_0$, we have 
\begin{flalign}
\| \boldsymbol{\lambda}^k - \boldsymbol{\lambda}^{k-1} \| \leq \rho \zeta.
\label{appendix eq: dual^k - dual^k-1 < rho tau}
\end{flalign}

Similar to (\ref{appendix eq: yk - y* < B(yk - y*)}), the full row rank of $\B$ implies $\| \y^k - \y^* \| \leq \Bar{M} \| \B(\y^k - \y^*) \|$. Then, plugging (\ref{appendix eq: dual^k - dual^k-1 < rho tau}) into (\ref{appendix eq: B(yk - y*) < A(xk -x*) + 1/rho (dual^k - dual^k-1)}),
we obtain that 
\begin{flalign}
 \| \y^k - \y^* \| \leq \Bar{M}(\|\hat{\A}\|+1) \zeta, 
 \label{appendix eq: (yk - y*) < M (A+1) tau}
\end{flalign}
for any $\hat{\x}^k \in \mathcal{V}_{\zeta, \hat{\x}}$ and $k \geq N_0$.
Combining (\ref{appendix eq: (zk - z*)^2 is lower than lambda_max(Q3) (xk - x*)^2}) and (\ref{appendix eq: (yk - y*) < M (A+1) tau}), we know that if $\hat{\x}^k \in \mathcal{V}_{\zeta, \hat{\x}}$ and $k \geq N_0$, then 
$(\y^k, \hat{\x}^k, \boldsymbol{\lambda}^k) \in \mathcal{V}_{\zeta} \subseteq \mathcal{V}$. 

Moreover, (\ref{appendix eq: L(yk+1,xk+1,dualk+1)-L(yk,xk,dualk) < 0}) and (\ref{appendix eq: limit of L_rho}) implies that $\mathcal{L}_{\rho, \epsilon}(\y^k, \hat{\x}^k, \boldsymbol{\lambda}^k) \geq l^*, \forall k \geq 1$. 
Besides, as $(\y^*, \hat{\x}^*, \boldsymbol{\lambda}^*)$ is a cluster point, we will obtain that $\exists N \geq N_0$, the following relations hold:
\begin{flalign}
  \begin{cases}
     \hat{\x}^N \in \mathcal{V}_{\zeta, \hat{\x}}
     \\
     l^* < \mathcal{L}_{\rho, \epsilon}(\y^N, \hat{\x}^N, \boldsymbol{\lambda}^N) < l^* + \eta
     \\
 \| \hat{\x}^N - \hat{\x}^* \| + 2 \sqrt{(\mathcal{L}_{\rho, \epsilon}(\y^N, \hat{\x}^N, \boldsymbol{\lambda}^N) - l^* ) / D} + \frac{C}{D} (\mathcal{L}_{\rho, \epsilon}(\y^N, \hat{\x}^N, \boldsymbol{\lambda}^N) - l^* ) < \zeta 
 \end{cases}
\label{item: relations of x^N}
\end{flalign}

Next, We will show that if $\hat{\x}^N \in \mathcal{V}_{\zeta, \hat{\x}}$ and $l^* < \mathcal{L}_{\rho, \epsilon}(\y^N, \hat{\x}^N, \boldsymbol{\lambda}^N) < l^* + \eta$ hold for some fixed $k \geq N_0$, then the following relation holds
\begin{flalign}
& \| \hat{\x}^{k+1} - \hat{\x}^k \| + \big(\| \hat{\x}^{k+1} - \hat{\x}^k \| - \| \hat{\x}^{k} - \hat{\x}^{k-1} \|\big) 
\leq 
\frac{C}{D} \bigg[ \varphi\big(\mathcal{L}_{\rho, \epsilon}(\y^k, \hat{\x}^k, \boldsymbol{\lambda}^k) - l^* \big) - \varphi\big(\mathcal{L}_{\rho, \epsilon}(\y^{k+1}, \hat{\x}^{k+1}, \boldsymbol{\lambda}^{k+1}) - l^* \big) \bigg].
\label{appendix eq: 2|x^k+1 - x^k| - |x^k - x^k-1| <= C/D phi(L^k) - phi(L^k+1)}
\end{flalign}
To prove (\ref{appendix eq: 2|x^k+1 - x^k| - |x^k - x^k-1| <= C/D phi(L^k) - phi(L^k+1)}), we utilize the fact that $\hat{\x}^k \in \mathcal{V}_{\zeta, \hat{\x}}, k \geq N_0$ implies that $(\y^k, \hat{\x}^k, \boldsymbol{\lambda}^k) \in \mathcal{V}_{\zeta, \hat{\x}} \subseteq \mathcal{V}$. And, combining with 
$l^* < \mathcal{L}_{\rho, \epsilon}(\y^k, \hat{\x}^k, \boldsymbol{\lambda}^k) < l^* + \eta$, we obtain that 
\begin{flalign}
\varphi'( \mathcal{L}_{\rho, \epsilon}(\y^k, \hat{\x}^k, \boldsymbol{\lambda}^k) - l^* ) \text{dist}(0, \partial \mathcal{L}_{\rho, \epsilon}(\y^k, \hat{\x}^k, \boldsymbol{\lambda}^k) ) \geq 1.
\label{appendix eq: KL property for (xk, yk, zk)}
\end{flalign}
Combining the relations (\ref{appendix eq: distance between 0 and partial L_rho over y^k+1 is lower than C x_residual}), (\ref{appendix eq: L(yk, xk, dualk) - L(yk+1, xk+1, dualk+1) <= D(xk+1 - xk)}) and (\ref{appendix eq: KL property for (xk, yk, zk)}), as well as the concavity of $\varphi$, we obtain that 
\begin{flalign}
& C \| \hat{\x}^k - \hat{\x}^{k-1} \| \cdot \big[ \varphi\big(\mathcal{L}_{\rho, \epsilon}(\y^k, \hat{\x}^k, \boldsymbol{\lambda}^k) - l^* \big) - \varphi\big(\mathcal{L}_{\rho, \epsilon}(\y^{k+1}, \hat{\x}^{k+1}, \boldsymbol{\lambda}^{k+1}) - l^* \big) \big] 
\label{appendix eq: C |x^k_residual| varphi_residual >= D x^k+1_residual}
\\
\geq & \text{dist}(0, \partial \mathcal{L}_{\rho, \epsilon}(\y^k, \hat{\x}^k, \boldsymbol{\lambda}^k) ) \cdot \big[ \varphi\big(\mathcal{L}_{\rho, \epsilon}(\y^k, \hat{\x}^k, \boldsymbol{\lambda}^k) - l^* \big) - \varphi\big(\mathcal{L}_{\rho, \epsilon}(\y^{k+1}, \hat{\x}^{k+1}, \boldsymbol{\lambda}^{k+1}) - l^* \big) \big] 
\nonumber 
\\
\geq &
\text{dist}(0, \partial \mathcal{L}_{\rho, \epsilon}(\y^k, \hat{\x}^k, \boldsymbol{\lambda}^k) ) \cdot  
\varphi'\big(\mathcal{L}_{\rho, \epsilon}(\y^k, \hat{\x}^k, \boldsymbol{\lambda}^k) - l^* \big) 
\cdot \big[ \mathcal{L}_{\rho, \epsilon}(\y^k, \hat{\x}^k, \boldsymbol{\lambda}^k) - \mathcal{L}_{\rho, \epsilon}(\y^{k+1}, \hat{\x}^{k+1}, \boldsymbol{\lambda}^{k+1}) \big]
\nonumber 
\\
\geq &
D \| \hat{\x}^{k+1} - \hat{\x}^{k} \|^2,
\nonumber 
\end{flalign}
for all such $k$. 
Taking square root on both sides of (\ref{appendix eq: C |x^k_residual| varphi_residual >= D x^k+1_residual}), and utilizing the fact that $a+b \geq 2 \sqrt{a b}$, then (\ref{appendix eq: 2|x^k+1 - x^k| - |x^k - x^k-1| <= C/D phi(L^k) - phi(L^k+1)}) is proved. 

We then prove $\forall k \geq N, \hat{\x}^k \in \mathcal{V}_{\zeta, \hat{\x}}$ holds. 
This claim can be proved through induction. Obviously it is true for $k=N$ by construction, as shown in (\ref{item: relations of x^N}). 
For $k=N+1$, we have 
\begin{flalign}
& \| \hat{\x}^{N+1} - \hat{\x}^{*} \| \leq \| \hat{\x}^{N+1} - \hat{\x}^{N} \| + 
\| \hat{\x}^{N} - \hat{\x}^{*} \|
\label{appendix eq: x^N+1 in V_mu} 
\\
\leq &
\sqrt{\big( \mathcal{L}_{\rho, \epsilon}(\y^N, \hat{\x}^N, \boldsymbol{\lambda}^N) - \mathcal{L}_{\rho, \epsilon}(\y^{N+1}, \hat{\x}^{N+1}, \boldsymbol{\lambda}^{N+1}) \big) / D} + \| \hat{\x}^{N} - \hat{\x}^{*} \|
\nonumber
\\
\leq & \sqrt{\big( \mathcal{L}_{\rho, \epsilon}(\y^N, \hat{\x}^N, \boldsymbol{\lambda}^N) - l^* \big) / D} + \| \hat{\x}^{N} - \hat{\x}^{*} \| < \zeta,
\nonumber
\end{flalign}
where the first inequality utilizes (\ref{appendix eq: L(yk, xk, dualk) - L(yk+1, xk+1, dualk+1) <= D(xk+1 - xk)}), and the last inequality follows the last relation in (\ref{item: relations of x^N}). 
Thus, $\forall k \geq N, \hat{\x}^k \in \mathcal{V}_{\zeta, \hat{\x}}$ holds. 

Next, we suppose that $\hat{\x}^N, \ldots, \hat{\x}^{N+t-1} \in \mathcal{V}_{\zeta, \hat{\x}}$ for some $t>1$, and we need to prove that $\hat{\x}^{N+t} \in \mathcal{V}_{\zeta, \hat{\x}}$ also holds, \ie, 
\begin{flalign}
&\| \hat{\x}^{N+t} - \hat{\x}^{*} \| 
\leq 
\| \hat{\x}^{N} - \hat{\x}^{*} \| + 
\| \hat{\x}^{N+1} - \hat{\x}^{N} \| + 
\sum_{i=1}^{t-1} \| \hat{\x}^{N+i+1} - \hat{\x}^{N+i} \|
\label{appendix eq: x^N+t in V_mu}
\\
= &
\| \hat{\x}^{N} - \hat{\x}^{*} \| + 2 \| \hat{\x}^{N+1} - \hat{\x}^{N} \| - \| \hat{\x}^{N+t} - \hat{\x}^{N+t-1} \| + 
\sum_{i=1}^{t-1} \bigg[ \| \hat{\x}^{N+i+1} - \hat{\x}^{N+i} \| + 
\big( \| \hat{\x}^{N+i+1} - \hat{\x}^{N+i} \| - \| \hat{\x}^{N+i} - \hat{\x}^{N+i-1} \| \big)  \bigg]
\nonumber 
\\
\leq &
 \| \hat{\x}^{N} - \hat{\x}^{*} \| + 2 \| \hat{\x}^{N+1} - \hat{\x}^{N} \| + \frac{C}{D} \sum_{i=1}^{t-1} \big[ \varphi^{N+i} - \varphi^{N+i+1} \big]
\nonumber 
\\
\leq &
\| \hat{\x}^{N} - \hat{\x}^{*} \| + 2 \| \hat{\x}^{N+1} - \hat{\x}^{N} \| + 
\frac{C}{D} \sum_{i=1}^{t-1} \varphi^{N+1}
\nonumber 
\\
\leq &
\| \hat{\x}^{N} - \hat{\x}^{*} \| + 2 \sqrt{ \frac{\mathcal{L}_{\rho, \epsilon}^N - \mathcal{L}_{\rho, \epsilon}^{N+1}}{D}} + \frac{C}{D} \sum_{i=1}^{t-1} \varphi^{N+1}
\nonumber 
\\
\leq &
\| \hat{\x}^{N} - \hat{\x}^{*} \| + 2 \sqrt{ \frac{\mathcal{L}_{\rho, \epsilon}^N - l^* }{D}} + \frac{C}{D} \sum_{i=1}^{t-1} \varphi^{N+1} < \zeta
\nonumber
\end{flalign}
where $\varphi^{N+i} = \varphi(\mathcal{L}_{\rho, \epsilon}(\y^{N+i}, \hat{\x}^{N+i}, \boldsymbol{\lambda}^{N+i})-l^*) $ and $\mathcal{L}_{\rho, \epsilon}^N = \mathcal{L}_{\rho, \epsilon}(\y^{N}, \hat{\x}^{N}, \boldsymbol{\lambda}^{N})$.
The second inequality follows from (\ref{appendix eq: 2|x^k+1 - x^k| - |x^k - x^k-1| <= C/D phi(L^k) - phi(L^k+1)}).
The fourth inequality follows from (\ref{appendix eq: L(yk, xk, dualk) - L(yk+1, xk+1, dualk+1) <= D(xk+1 - xk)}). 
The fifth inequality utilizes the fact that $\mathcal{L}_{\rho, \epsilon}^{N+1} > l^*$, and the last inequality follows from the last relation in (\ref{item: relations of x^N}). 
Thus, $\hat{\x}^{N+k} \in \mathcal{V}_{\zeta, \hat{\x}}$ holds. We have proved that $\forall k \geq N, \hat{\x}^k \in \mathcal{V}_{\zeta, \hat{\x}}$ holds by induction.

Then, according to $\forall k \geq N, \hat{\x}^k \in \mathcal{V}_{\zeta, \hat{\x}}$, we can sum both sides of
(\ref{appendix eq: 2|x^k+1 - x^k| - |x^k - x^k-1| <= C/D phi(L^k) - phi(L^k+1)}) from 
$k=N$ to $\infty$, to obtain that
\begin{flalign}
\sum_{k=N}^{\infty} \| \hat{\x}^{k+1} - \hat{\x}^{k} \| \leq \frac{C}{D} \varphi^N + \| \hat{\x}^{N} - \hat{\x}^{N-1} \| < \infty, 
\label{appendix eq: sum |x^k+1 - x^k| <= C/D varphi^N + | x^N - x^N-1 |}
\end{flalign}
which implies that $\sum_{k=1}^{\infty} \| \hat{\x}^{k+1} - \hat{\x}^k \| < \infty$ holds. Thus $\{ \hat{\x}^k \}$ converges. 
The convergence of $\{ \y^k \}$ follows from $\B \y^{k+1} = \hat{\A} \hat{\x}^{k+1} + \frac{1}{\rho}(\boldsymbol{\lambda}^{k+1} -\boldsymbol{\lambda}^k)$ in (\ref{appendix eq: optimality 3}) and (\ref{appendix eq: x_residual + y-residual + dual-residual converges to 0}), as well as the surjectivity of $\B$ (\ie, full row rank). 
The convergence of $\{ \boldsymbol{\lambda}^k \}$ follows from $\nabla \hat{f}(\hat{\x}^{k+1}) = - \hat{\A}^\top \boldsymbol{\lambda}^{k+1}$ in (\ref{appendix eq: optimality 2}) and the surjectivity of $\hat{\A}$  (\ie, full row rank). 
Consequently, $\{ \y^k, \hat{\x}^k, \boldsymbol{\lambda}^k \}$ converges to the cluster point $(\y^*, \hat{\x}^*, \boldsymbol{\lambda}^*)$. 
The conclusion that $(\y^*, \hat{\x}^*)$ is the KKT point of Problem (\ref{appendix eq: LS-LP(x,y) with perturbation}) has been proved in Section \ref{sec: subsec Convergence of residual}. 


\subsection{$\epsilon$-KKT Point of the Original LS-LP Problem}
\label{sec: Stationary and epsilon-Feasible Point of the Original LS-LP Problem}

\begin{proposition}
The globally converged solution $(\y^*, \x^*, \boldsymbol{\lambda}^*)$ produced by the ADMM algorithm for the perturbed LS-LP problem (\ref{appendix eq: LS-LP(x,y) with perturbation}) is the $\epsilon$-KKT solution to the original LS-LP problem (\ref{appendix eq: LS-LP(x,y)}). 
\end{proposition}

\begin{proof}
The globally converged solution $(\y^*, \hat{\x}^*, \boldsymbol{\lambda}^*)$ to the perturbed LS-LP problem (\ref{appendix eq: LS-LP(x,y) with perturbation}) satisfies the following relations:
\begin{flalign}
 \B^\top \boldsymbol{\lambda}^* \in \partial h(\y^*), ~
 \nabla \hat{f}(\hat{\x}^{*}) = - \hat{\A}^\top \boldsymbol{\lambda}^*, ~
 \hat{\A} \hat{\x}^* = \B \y^*.
\end{flalign}
Recalling the definitions $\hat{\A} = [\A, \epsilon \I]$, $\hat{\x} = [\x; \bar{\x}]$ and $\hat{f}(\hat{\x}) = f(\x) + \frac{\epsilon}{2} \hat{\x}^\top \hat{\x}$, the above relations imply that 
\begin{flalign}
 & \nabla \hat{f}(\hat{\x}^{*}) + \hat{\A}^\top \boldsymbol{\lambda}^* 
 = \nabla f(\x^{*}) + \A^\top \boldsymbol{\lambda}^* + \epsilon \x^* 
 = \boldsymbol{0} 
 ~ \Rightarrow ~
 \| \nabla f(\x^{*}) + \A^\top \boldsymbol{\lambda}^* \| = \epsilon \| \x^* \| = O(\epsilon), 
 \\
 & \hat{\A} \hat{\x}^* + \B \y^* = \A \x^* + \epsilon \bar{\x}^* + \B \y^* = 0 
 ~ \Rightarrow ~
 \| \A \x^* + \B \y^* \| = \| \epsilon \bar{\x}^* \| = O(\epsilon), 
\end{flalign}
where we utilize the boundedness of $\{\hat{\x}^*\}$. 
Thus, according to Definition \ref{definition: epsilon kkt solution}, the globally converged point $(\y^*, \x^*)$ is the $\epsilon$-KKT solution to the original LS-LP problem (\ref{appendix eq: LS-LP(x,y)}). 
\end{proof}

\subsection{Convergence Rate}
\label{sec: convergence rate}

\begin{lemma}
 Firstly, without loss of generality, we can assume that $l^* = \mathcal{L}_{\rho, \epsilon}(\y^*, \hat{\x}^*, \boldsymbol{\lambda}^*) = 0$ (\eg, one can replace $l_k = \mathcal{L}_{\rho, \epsilon}(\y^k, \hat{\x}^k, \boldsymbol{\lambda}^k)$ by $l_k - l^*$). 
We further assume that $\mathcal{L}_{\rho, \epsilon}$ has the KL property at $(\y^*, \hat{\x}^*, \boldsymbol{\lambda}^*)$ with the concave function $\varphi(s) = c s^{1-p}$, where $p \in [0, 1), c>0$. 
Consequently, we can obtain the following inequalities:
\begin{enumerate}
    \item[(i)] if $p=0$, then $\{(\y^k, \hat{\x}^k, \boldsymbol{\lambda}^k)\}_{k = 1, \ldots, \infty}$ can converge in finite steps;
    \item[(ii)] If $p \in (0,\frac{1}{2}]$, then there exist $c>0$ and $\tau \in (0,1)$ such that $\| \hat{\x}^{k+1} - \hat{\x}^k \| \leq c \tau^k$;
    \item[(iii)] $p \in (\frac{1}{2},1)$, then there exist $c>0$ such that $\| \hat{\x}^{k+1} - \hat{\x}^k \| \leq c k^{-\frac{1-p}{2p -1}} $.
\end{enumerate}
\label{appendix lemma: xk - x* < }
\end{lemma}

\begin{proof}
\noindent
(\textit{i}) If $p=0$, we define a subset $H = \{k \in \mathbb{N}: \hat{\x}_{k} \neq \hat{\x}_{k+1} \}$. If $k \in H$ is sufficiently large, then these exists $C_3>0$ such that
\begin{flalign}
 \| \hat{\x}^{k+1} - \hat{\x}^k \|^2 \geq C_3 > 0.
\end{flalign}
Combining with (\ref{appendix eq: L(yk, xk, dualk) - L(yk+1, xk+1, dualk+1) <= D(xk+1 - xk)}), we have
\begin{flalign}
 l_k - l_{k+1} \geq D \| \hat{\x}^{k+1} - \hat{\x}^k \|^2 \geq C_3 D > 0.
\end{flalign}
If the subset $H$ is infinite, then it will contradict to the fact that $l_k - l_{k+1} \rightarrow 0$ as $k\rightarrow \infty$. Thus, $H$ is a finite subset, leading to the conclusion that $\{\hat{\x}^k\}_{k \in \mathbb{N}}$ will converge in finite steps. Recalling the relationships between $\hat{\x}_k$ and $\y_k, \boldsymbol{\lambda}_k$ (see the descriptions under (\ref{appendix eq: sum |x^k+1 - x^k| <= C/D varphi^N + | x^N - x^N-1 |})), we also obtain that $\{ \y_k, \boldsymbol{\lambda}_k \}_{k \in \mathbb{N}}$ converges in finite steps. 

By defining $\bigtriangleup_k = \sum_k^{\infty} \| x^k+1 - x^k \|$, 
the inequality (\ref{appendix eq: sum |x^k+1 - x^k| <= C/D varphi^N + | x^N - x^N-1 |}) can be rewritten as follows
\begin{flalign}
 \bigtriangleup_k \leq \frac{C}{D} \varphi(l_k) + (\bigtriangleup_{k-1} - \bigtriangleup_k) < \infty.
 \label{appendix: eq uptriangle_k < C/D varphi(lk) + (uptriange_(k-1) - uptriangle_k)}
\end{flalign}
Besides, the KL property and $l^* = 0$ give that 
\begin{flalign}
  \varphi'(l_k) \text{dist}(0, \partial (l_k)) = c (1-p) l_k^{1-p} \text{dist}(0, \partial (l_k)) \geq 1 ~
 \Rightarrow ~
 l_k^p \leq c(1-p) \text{dist}(0, \partial (l_k)).
\end{flalign}
Combining with (\ref{appendix eq: distance between 0 and partial L_rho over y^k+1 is lower than C x_residual}), we obtain 
\begin{flalign}
  l_k^p \leq c(1-p) C (\bigtriangleup_{k-1} - \bigtriangleup_k) 
  ~ \Rightarrow ~ 
  \varphi(l_k) = c l_k^{1-p} \leq c (c(1-p) C)^{\frac{1-p}{p}} (\bigtriangleup_{k-1} - \bigtriangleup_k)^{\frac{1-p}{p}} = C_1 (\bigtriangleup_{k-1} - \bigtriangleup_k)^{\frac{1-p}{p}}.
  \label{appendix: eq varphi < C (uptriange_(k-1) - uptriangle_k)^((1-p)/p)}
\end{flalign}
Then, inserting (\ref{appendix: eq varphi < C (uptriange_(k-1) - uptriangle_k)^((1-p)/p)}) into (\ref{appendix: eq uptriangle_k < C/D varphi(lk) + (uptriange_(k-1) - uptriangle_k)}), we obtain 
\begin{flalign}
 \bigtriangleup_k \leq C_2 (\bigtriangleup_{k-1} - \bigtriangleup_k)^{\frac{1-p}{p}} + (\bigtriangleup_{k-1} - \bigtriangleup_k) < \infty.
 \label{appendix: eq uptriangle_k < C_hat (uptriange_(k-1) - uptriangle_k)^((1-p)/p) + (uptriange_(k-1) - uptriangle_k)}
\end{flalign}

\noindent
(\textit{ii}) 
If $p \in (0, \frac{1}{2}]$, then $\frac{1-p}{p}\geq 1$. Besides, since $(\bigtriangleup_{k-1} - \bigtriangleup_k) \rightarrow 0$ when $k \rightarrow \infty$, there exists an integer $K_0$ such that $(\bigtriangleup_{k-1} - \bigtriangleup_k) < 1$, leading to that $(\bigtriangleup_{k-1} - \bigtriangleup_k)^{\frac{1-p}{p}} \leq  (\bigtriangleup_{k-1} - \bigtriangleup_k)$. Inserting it into (\ref{appendix: eq uptriangle_k < C_hat (uptriange_(k-1) - uptriangle_k)^((1-p)/p) + (uptriange_(k-1) - uptriangle_k)}), we obtain that 
\begin{flalign}
 \bigtriangleup_k \leq (C_2+1) (\bigtriangleup_{k-1} - \bigtriangleup_k)
 ~ \Rightarrow ~ 
 \bigtriangleup_k \leq C_3 (\bigtriangleup_{k-1} - \bigtriangleup_k)
 ~ \Rightarrow ~ 
 \bigtriangleup_k \leq \frac{C_3}{1+C_3} \bigtriangleup_{k-1} = \tau \bigtriangleup_{k-1},
 ~ \text{with}~ \tau \in (0,1), \forall k > K_0.
 \label{appendix: eq uptriangle_k < C_prime (uptriange_(k-1) - uptriangle_k)}
\end{flalign}
It is easy to deduce that $\bigtriangleup_k \leq (\bigtriangleup_{K_0} \tau^{-K_0}) \tau^k = \frac{c}{2} \tau^k$, with $c$ being a positive constant. Note that $k$ in $\tau^k$ indicates $k$ power of $\tau$, rather than the iteration index.
Combining with $\| \hat{\x}^k - \hat{\x}^* \| \leq \bigtriangleup_k$, it is easy to obtain that $\| \hat{\x}^k - \hat{\x}^* \| \leq \frac{c}{2} \tau^k$ with $\tau \in (0,1)$ and $c$ being a positive constant. 
Then, we have 
\begin{flalign}
 \| \hat{\x}^{k+1} - \hat{\x}^k \| \leq \| \hat{\x}^k - \hat{\x}^* \| + \| \hat{\x}^{k+1} - \hat{\x}^* \| \leq \frac{c}{2} (\tau^{k+1} + \tau^k) \leq c \tau^k. 
\end{flalign}

\noindent
(\textit{iii}) 
If $p \in (\frac{1}{2}, 1)$, then $\frac{1-p}{p}<1$. 
Then, it is easy to obtain that $(\bigtriangleup_{k-1} - \bigtriangleup_k)^{\frac{1-p}{p}} > (\bigtriangleup_{k-1} - \bigtriangleup_k)$. Inserting it into (\ref{appendix: eq uptriangle_k < C_hat (uptriange_(k-1) - uptriangle_k)^((1-p)/p) + (uptriange_(k-1) - uptriangle_k)}), we obtain that \begin{flalign}
 \bigtriangleup_k \leq (C_2+1) (\bigtriangleup_{k-1} - \bigtriangleup_k)^{\frac{1-p}{p}}
 ~ \Rightarrow ~ 
 \bigtriangleup_k \leq C_3 (\bigtriangleup_{k-1} - \bigtriangleup_k)^{\frac{1-p}{p}}
 ~ \Rightarrow ~ 
 \bigtriangleup_k^{\frac{p}{1-p}} \leq C_4 (\bigtriangleup_{k-1} - \bigtriangleup_k), 
 ~ \forall k > K_0.
 \label{appendix: eq uptriangle_k < C_prime (uptriange_(k-1) - uptriangle_k)^((1-p)/p)}
\end{flalign}
It has been studied in Theorem 2 of \cite{attouch2009convergence} that the above inequality can deduce 
$\bigtriangleup_k \leq \frac{c}{2} k^{- \frac{1-p}{2p-1}}$, with $c$ being a positive constant. 
Since $\| \hat{\x}^k - \hat{\x}^* \| \leq \bigtriangleup_k$, we have that $\| \hat{\x}^k - \hat{\x}^* \| \leq \frac{c}{2} k^{- \frac{1-p}{2p-1}}$. 
Then, we have 
\begin{flalign}
 \| \hat{\x}^{k+1} - \hat{\x}^k \| \leq \| \hat{\x}^k - \hat{\x}^* \| + \| \hat{\x}^{k+1} - \hat{\x}^* \| \leq \frac{c}{2} \big(k^{- \frac{1-p}{2p-1}} + (k+1)^{- \frac{1-p}{2p-1}} \big) \leq c k^{- \frac{1-p}{2p-1}}. 
\end{flalign}
\end{proof}

\begin{proposition}
We adopt the same assumptions in Lemma \ref{appendix lemma: xk - x* < }. Then, 
\begin{enumerate}
     \item[(i)] If $p  =0$, then 
     we will obtain the $\epsilon$-KKT solution to the LS-LP problem in finite steps. 
    \item[(ii)] If $p \in (0, \frac{1}{2}]$, then 
     we will obtain the $\epsilon$-KKT solution to the LS-LP problem in at least $O\big(\log_{\frac{1}{\tau}}(\frac{1}{\epsilon})^2\big)$ steps. 
    \item[(iii)] If $p \in ( \frac{1}{2}, 1)$, then 
    we will obtain the $\epsilon$-KKT solution to the LS-LP problem in at least $ O\big( (\frac{1}{\epsilon})^{\frac{4p-2}{1-p}}\big)$ steps.
\end{enumerate}
\end{proposition}

\begin{proof}
The conclusion (i) directly holds from Lemma \ref{appendix lemma: xk - x* < }(i).

According to the optimality condition (\ref{appendix eq: optimality 1}), we have 
\begin{flalign}
 & \text{dist}\big( \B^\top \boldsymbol{\lambda}^{k+1}, \partial h(\y^{k+1}) \big) = \| \rho \B^\top \hat{\A} ( \hat{\x}^{k+1} - \hat{\x}^k) \|_2
 \nonumber
 \\
 \Rightarrow & ~
 \text{dist}^2\big( \B^\top \boldsymbol{\lambda}^{k+1}, \partial h(\y^{k+1})\big) = \| \hat{\x}^{k+1} - \hat{\x}^{k} \|_{\rho^2 \hat{\A}^\top \B \B^\top \hat{\A}}^2 \leq \xi_{\text{max}}(\rho^2 \hat{\A}^\top \B \B^\top \hat{\A}) \| \hat{\x}^{k+1} - \hat{\x}^{k} \|_2^2 
 = O(\frac{1}{\epsilon^2}) \| \hat{\x}^{k+1} - \hat{\x}^{k} \|_2^2
 \nonumber
 \\
 \Rightarrow & ~
 \text{dist}\big( \B^\top \boldsymbol{\lambda}^{k+1}, \partial h(\y^{k+1}) \big) 
 \leq O(\frac{1}{\epsilon})  \cdot \| \hat{\x}^{k+1} - \hat{\x}^{k} \|_2
 \label{eq: | dist(By, partial h(y)) |^2 < O(epsilon)}
\end{flalign}

According to the optimality condition (\ref{appendix eq: optimality 3}) and the relation (\ref{appendix eq: dual_residual^2 <= x_residual^2}), we obtain that 
\begin{flalign}
   \| \hat{\A} \hat{\x}^{k+1} - \B \y^{k+1} \|_2 
   = 
    \frac{1}{\rho}\| \boldsymbol{\lambda}^{k +1} - \boldsymbol{\lambda}^{k} \|_2 
    \leq
    \frac{1}{\rho} \| \hat{\x}^{k+1} - \hat{\x}^k \|_2
    \leq
    \epsilon \cdot \| \hat{\x}^{k+1} - \hat{\x}^k \|_2 \leq O(\frac{1}{\epsilon})  \cdot \| \hat{\x}^{k+1} - \hat{\x}^{k} \|_2.
    \label{eq: | Ax - By | < O(epsilon)}
\end{flalign}

According to Lemma \ref{appendix lemma: xk - x* < }, we have 
\begin{enumerate}
    \item[(ii)] If $p \in (0, \frac{1}{2}]$, then 
    \[
    O(\frac{1}{\epsilon})  \cdot \| \hat{\x}^{k+1} - \hat{\x}^{k} \|_2 \leq O(\frac{1}{\epsilon}) \tau^k \leq O(\epsilon) 
    ~ \Rightarrow ~
    k \geq O\big(\log_{\frac{1}{\tau}}(\frac{1}{\epsilon})^2\big),
    \]
    which means that when $k \geq O\big(\log_{\frac{1}{\tau}}(\frac{1}{\epsilon})^2\big)$, we will obtain the $\epsilon$-KKT solution to the perturbed LS-LP problem, \ie, the $\epsilon$-KKT solution to the original LS-LP problem. 
    \item[(iii)] If $p \in ( \frac{1}{2}, 1)$, then 
    \[
    O(\frac{1}{\epsilon})  \cdot \| \hat{\x}^{k+1} - \hat{\x}^{k} \|_2 \leq O(\frac{1}{\epsilon}) k^{- \frac{1-p}{2p-1}} \leq O(\epsilon) 
    ~ \Rightarrow ~
    k \geq O\big( (\frac{1}{\epsilon})^{\frac{4p-2}{1-p}}\big),
    \]
    which means that when $k \geq O\big( (\frac{1}{\epsilon})^{\frac{4p-2}{1-p}}\big)$, we will obtain the $\epsilon$-KKT solution to the perturbed LS-LP problem, \ie, the $\epsilon$-KKT solution to the original LS-LP problem.
\end{enumerate}
\end{proof}

\twocolumn

\bibliographystyle{spmpsci}
\bibliography{BYWU_bib_IJCV}

\end{document}